\documentclass[10pt, conference, letterpaper]{IEEEtran}

\usepackage{algorithm}
\usepackage{algpseudocode}
\usepackage{amsfonts}
\usepackage{amsmath}
\usepackage{amssymb}
\usepackage{amsthm}
\usepackage[ansinew]{inputenc}
\usepackage[table,xcdraw]{xcolor}
\usepackage{mathtools}
\usepackage{graphicx}
\usepackage{caption}
\usepackage{subcaption}
\usepackage{import}
\usepackage{multirow}
\usepackage{cite}
\usepackage[export]{adjustbox}
\usepackage{breqn}
\usepackage{mathrsfs}
\usepackage{acronym}
\usepackage{setspace}
\usepackage{bm}
\usepackage{stackengine}
\usepackage{booktabs}
\usepackage{siunitx}
\usepackage{listings}
\usepackage{csvsimple}
\usepackage{wrapfig}
\usepackage{epsfig}
\usepackage{graphicx}
\usepackage{epstopdf}
\usepackage[hyphens]{url}
\usepackage{hyperref}
\hypersetup{breaklinks=true}

\lstdefinelanguage{BibTeX}
  {keywords={%
      @article,@book,@collectedbook,@conference,@electronic,@ieeetranbstctl,%
      @inbook,@incollectedbook,@incollection,@injournal,@inproceedings,%
      @manual,@mastersthesis,@misc,@patent,@periodical,@phdthesis,@preamble,%
      @proceedings,@standard,@string,@techreport,@unpublished%
      },
   comment=[l][\itshape]{@comment},
   sensitive=false,
  }

\usepackage{listings}

\renewcommand{\thetable}{\arabic{table}}

\graphicspath{{./figures/}}
\usepackage{soul}
\newtheorem{theorem}{Theorem}

\title{The Cost of Training Machine Learning Models over Distributed Data Sources}
\author{Elia Guerra, Francesc Wilhelmi, Marco Miozzo, Paolo Dini
\thanks{Elia Guerra, Marco Miozzo, and Paolo Dini are with Centre Tecnol\`ogic de Telecomunicacions de Catalunya (CTTC/CERCA). Francesc Wilhelmi is with Nokia Bell-Labs, Stuttgart, Germany. Corresponding author: eguerra@cttc.cat. This publication has been partially funded by the Spanish project PID2020-113832RB-C22(ORIGIN)/MCIN/AEI/10.13039/50110001103, European Union Horizon 2020 research and innovation programme under Grant Agreement No. 953775 (GREENEDGE) and the grant CHIST-ERA-20-SICT-004 (SONATA) by PCI2021-122043-2A/AEI/10.13039/501100011033}
}

\IEEEoverridecommandlockouts

\newcounter{remark}[section]

\begin{document}

\maketitle

\begin{abstract}
Federated learning is one of the most appealing alternatives to the standard centralized learning paradigm, allowing a heterogeneous set of devices to train a machine learning model without sharing their raw data. 
However, it requires a central server to coordinate the learning process, thus introducing potential scalability and security issues. In the literature, server-less federated learning approaches like gossip federated learning and blockchain-enabled federated learning have been proposed to mitigate these issues. 
In this work, we propose a complete overview of these three techniques proposing a comparison according to an integral set of performance indicators, including model accuracy, time complexity, communication overhead, convergence time, and energy consumption. 
An extensive simulation campaign permits to draw a quantitative analysis considering both feedforward and convolutional neural network models.  
Results show that gossip federated learning and standard federated solution are able to reach a similar level of accuracy, and their energy consumption is influenced by the machine learning model adopted, the software library, and the hardware used.
Differently, blockchain-enabled federated learning represents a viable solution for implementing decentralized learning with a higher level of security, at the cost of an extra energy usage and data sharing.
Finally, we identify open issues on the two decentralized federated learning implementations and provide insights on potential extensions and possible research directions in this new research field.
\end{abstract}

\IEEEkeywords
blockchain, decentralized learning, edge computing, energy efficiency,  federated learning, machine learning.
\endIEEEkeywords

\section{Introduction}
\label{sec:introduction}

learning (ML) models, and in particular deep neural networks, require a substantial amount of data and computational power that might not be available on a single machine. As a consequence, ML operations are normally run at cloud servers (or data centers), where batteries of powerful processing units enable short training and inference computation times. However, training ML models in a data center requires moving data from the information sources (e.g., edge devices) to the central system. This approach runs into several issues:
\begin{itemize}
	\item \textbf{Communication overhead.} Nowadays, the huge pervasiveness of mobile services, devices, and network infrastructures makes data sources mainly distributed. As testified recently by the Ericsson Mobility Report~\cite{ericsson_mobility}, mobile network data traffic grew exponentially over the last $10$ years, with a remarkable increase of $42 \%$ between Q3 2020 and Q3 2021. Mobile data traffic is projected to grow by over $4$ times to reach $288$~EB per month by 2027~\cite{ericsson_mobility}. Moving such a big amount of data from distributed sources to a central location for ML operations may create network congestion and service outage.
	\item \textbf{Latency.} In several real-life scenarios, transmitting data requires a stable and reliable connection to minimize latency and ensure updated models, which cannot be always guaranteed. For example, minimizing communication latency in connected vehicles is essential to guarantee road safety~\cite{8584062}.
	\item \textbf{Energy consumption.} Running ML models in cloud data centers consumes a significant amount of energy and cannot be considered sustainable from an environmental perspective. As reported in~\cite{noauthor_ai_2018}, from 2012 to 2018, the computations required for training a deep learning (DL) model have been doubling every 3.4 months, with an estimated increase of $300000$x. Estimates show that training a state-of-the-art natural language processing model produces more $\text  {CO}_2$ than an average car in one year lifetime~\cite{strubell2019energy}.
	\item \textbf{Privacy.} With the growing awareness of data privacy and security, it is often undesirable, or even unfeasible, to collect and centralize users' data~\cite{zhang2022introduction}. For instance, a single hospital may not be able to train a high-quality model for a specific task on its own (due to the lack of data), but it cannot share raw data due to various policies or regulations on privacy~\cite{li2019survey}. Another example could be the case of a mobile user that would like to employ a good next-word predictor model without sharing his/her private historical text data.
\end{itemize}

\subsection{Edge AI and Federated Learning}
To address the challenges that stem from cloud-based centralized ML, edge computing pushes cloud services to the network edge and enables distributed ML operations, i.e., the so-called edge intelligence~\cite{chen2021distributed}. In particular, AI on Edge~\cite{deng2020edge} is the paradigm of running AI models with a device-edge-cloud synergy. It allows to relax the massive communication requirements and privacy of cloud-based ML operations~\cite{zhou2019edge}. Moreover, distributing ML computation over the edge has been demonstrated to save up to the $25\%$ of the energy consumption~\cite{Ahvar2019}. In fact, data may be directly processed at the edge with smaller and more energy efficient devices (no need of air conditioning systems) and the energy cost related to communication is limited due to unnecessary data transmission. 

Among the several training paradigms enabled by edge intelligence, federated learning (FL) has emerged as a popular solution by providing low communication overhead, enhanced user privacy, and security to distributed learning~\cite{konevcny2016federated}. With FL, the ML model is trained cooperatively by edge devices without sharing local data, but exchanging only model parameters. 
The usual implementation envisages an iterative procedure whereby a central server collects local updates from the clients (e.g., edge devices) and returns an aggregated global model. 
In the rest of the paper, we refer to centralized FL (CFL) to the traditional server-dependent FL scenario. In this setting, the server has to wait for all the clients before returning a new global update. Therefore, high network latency, unreliable links, or straggled clients may slow down the training process and even worsen model accuracy~\cite{xie2019asynchronous}. 
In addition, the central server represents a single point of failure, i.e., if it is unreachable due to network problems or an attack, the training process cannot continue. Furthermore, it may also become a bottleneck when the number of clients is very large~\cite{kairouz2019advances}. 

Decentralized and server-less solutions for federated learning have been introduced in the literature, mainly to overcome the single point of failure and the security problems~\cite{mothukuri2021survey, barbieri2022decentralized}. 
In~\cite{lalitha2018fully} a decentralized FL mechanism was proposed by enabling one-hop communication among FL clients. Similarly, gossip FL (GFL) extends device-to-device (D2D) communications to compensate for the lack of an orchestrating central server~\cite{ormandi2013gossip,giaretta2019gossip}. It guarantees a low communication overhead thanks to the reduced number of messages~\cite{miozzo2021distributed}. 

Beyond, we find more sophisticated proposals, like blockchain-enabled federated learning (BFL), which adopts blockchain to share FL information among devices, thus removing the figure of the orchestrating central server. In this way, blockchain removes the single point of failure for the sake of openness and decentralization and provides enhanced security via tampered-proof properties~\cite{wilhelmi2021blockchain}. 

\subsection{Contributions}
Despite in the literature it is possible to find papers comparing classical centralized learning in data center with CFL~\cite{qiu2020can},~\cite{savazzi2022energy}, a comparison among the different federated learning approaches (centralized versus decentralized) is still missing. 
In this work, we aim to fill this gap and, thus, we focus on two of the most popular and widely adopted approaches for decentralizing FL: GFL and BFL. In particular, we provide a comprehensive analysis of both methods and compare them to traditional FL, i.e., CFL. Note that we combine standard performance indicators for ML models, i.e., accuracy, with indicators that quantify the efficiency of these algorithms, i.e., time complexity, communication overhead, convergence time, and energy consumption. 
With our comparison under fair conditions, we would like to provide the research community with a complete overview of the three approaches together with all the information to choose the best model according to the specific use cases.

The contributions of this paper may be summarized as follows:
\begin{itemize}
	\item We overview the traditional FL setting and delve into two approaches for decentralizing it. They are selected since are two of the most popular in the literature and are kind of diverging into two completely different solutions, which are based on gossip communication and blockchain technology, respectively.
	\item We provide a thorough analysis to derive the running time complexity, the communication overhead and the convergence time of each overviewed mechanism for FL, including CFL, BFL, and GFL. 
	\item We provide an energy model to measure the energy consumption of each solution, based on the associated communication and computation overheads. 
	\item We assess the performance of each method (CFL, GFL, and BFL) through extensive simulations on widely used TensorFlow libraries~\cite{tensorflow2015-whitepaper}.
	\item We delve into the open aspects of decentralized FL, providing insights on potential extensions, considerations, and software implementations for GFL and BFL.
\end{itemize}

The rest of the paper is structured as follows: Section~\ref{sec:relatedwork} reviews the related work. Section~\ref{sec:background} describes the three studied algorithms (CFL, BFL, and GFL). Section~\ref{sec:theoretical_analysis} analyzes their time complexity, the communication cost, introduces the communication model and the convergence time. Section~\ref{sec:energy_communication_model} provides the energy model used in this paper. Then, Section~\ref{sec:exp_settings} compares the three mechanisms through simulation results. 
In Section~\ref{sec:open-aspects}, we provide some open issues of GFL and BFL, and we discuss the correspondent optimizations and future research directions.
Finally, Section~\ref{sec:conclusions} concludes the paper with final remarks.

\section{Related Work}
\label{sec:relatedwork}

Distributing and decentralizing ML operations at the edge has been embraced as an appealing solution for addressing the issues of centralization (connectivity, privacy, and security)~\cite{verbraeken2020survey}. With FL, different devices collaborate to train an ML model by sharing local updates obtained from local and private data. The traditional FL algorithm (FedAvg), referred to as CFL in this paper, is introduced in~\cite{mcmahan2017communication}. In~\cite{konevcny2016federated}, the authors propose techniques to improve its communication efficiency.  Nevertheless, CFL still requires a central server responsible for clients orchestration and model aggregation. The star topology is a weak aspect of CFL, since the central entity represents a single point of failure, it may limit the number of devices that can participate in the training process, augments the communication cost of the learning process,  
and presents privacy concerns~\cite{li2020federated, kairouz2019advances}.   

To address these challenges, decentralized federated learning has been proposed in~\cite{lalitha2018fully}. The authors present a fully decentralized model, in which each device can communicate only with its one-hop neighbors. The authors also provide a theoretical upper bound on the mean square error. 
Ormandi et al.~\cite{ormandi2013gossip} introduce gossip FL, a generic approach for peer-to-peer (P2P) learning on fully distributed data, i.e., every device has only a single input sample to process at each round.
The same algorithm has been tested in~\cite{giaretta2019gossip} under real-world conditions, i.e., devices have multiple input samples available (rather than only one point, as originally stated in~\cite{ormandi2013gossip}), restricted communication topology and,
heterogeneous communication and computation capabilities.

Another prominent solution to decentralize FL is blockchain-enabled FL~\cite{kim2019blockchained, majeed2019flchain, 8905038}. A blockchain system allows clients to submit and retrieve model updates without the central server. Additionally, the usage of blockchain guarantees security, trust, privacy, and traceability, however, it introduces delays due to the distributed ledger technology. An analysis of end-to-end latency and the effects of blockchain parameters on the training procedure of BFL is proposed in~\cite{wilhelmi2021blockchain}.

In the literature, there exist some comparisons across FL techniques. The authors of~\cite{hegedHus2021decentralized} compare GFL and CFL with a logistic regression model 
in terms of convergence time, proportion of the misclassified examples in the test set (0-1 error), and used communication resources. 
When nodes have a random subset of the learning samples, GFL performance is comparable with CFL; instead, CFL converges faster when a node has only labels from one class.
Another comparison is proposed in~\cite{nilsson2018performance}, where 
the performance of FL algorithms that require a central server, e.g., FedAvg and Federated Stochastic Reduced Gradient are analyzed. Results show that FedAvg achieves the highest accuracy among the FL algorithms regardless of how data are partitioned. In addition, the comparison between FedAvg and the standard centralized algorithm shows that they are equivalent when independent and identically distributed (IID) datasets are used. 

In~\cite{miozzo2021distributed}, the authors compare GFL 
with the standard centralized data center based architecture in terms of accuracy and energy consumption for two radio access network use-cases.
To achieve this goal, they use the machine learning emission calculator~\cite{lacoste2019quantifying} and green algorithms~\cite{lannelongue2021green}.
In~\cite{qiu2020can}, the authors compare centralized data center based learning and CFL in terms of carbon footprint using different datasets. The assessment is done by sampling the CPU and GPU power consumption. 
In~\cite{savazzi2022energy}, the authors propose a framework to evaluate the energy consumption and the carbon footprint of distributed ML models with focus on industrial Internet of Things applications. The paper identifies specific requirements on the communication network, dataset and model size to guarantee the energy efficiency of CFL over centralized learning approaches, i.e., bounds on the local dataset or model size.
Differently from our work, the authors do not consider Blockchain-enabled FL and evaluate the algorithm performance in scenarios with limited number of devices (i.e., 100). Moreover, we also empirically measure the energy consumption of the devices based on the real load of the computations realized during the training phase; we provide a communication model to estimate overheads and convergence time.
Additionally, here we introduce an analysis on the computational complexity of the three federated algorithms under study.

In summary, in this paper, we bridge the gap in the literature by providing a thorough comparison including performance analysis and cost of the different federated approaches listed above, i.e., CFL, BFL, and GFL. 
Differently from the other works in the literature, we combine standard metrics, i.e., accuracy, with indicators of the efficiency of these algorithms, i.e., computational complexity, communication overhead, convergence time and energy consumption. 
Our final aim is to contribute to the development of Green AI~\cite{schwartz2020green}.

\section{Federated Learning Implementations}
\label{sec:background}

Let us consider a set of $N$ clients (or devices) $\mathcal{N} = \{1, ..., N\}$ with their datasets $D_1, ..., D_{N}$. Each local dataset $D_i, \forall i \in \mathcal{N}$, contains pairs $(x_i, y_i)$, where $x_i$ is the feature vector, and $y_i$ its true label. The goal of a federated setting is to train a global model (e.g., a set of weights $w$), that minimizes the weighted global loss function: 
\begin{equation}
    \ell = \sum_{i = 1}^{N} \ell_i(w, x_i, y_i),
\end{equation}
where $\ell_i$ is the local loss experienced by client $i$. In this scenario, devices do not share raw local data with other devices. Instead, they exchange model parameter updates, computed during several iterations by training the global model on local data.
In this paper, we study three different implementations to solve the federated problem stated above, namely: CFL, BFL, and GFL. The investigated solutions are depicted in Fig.~\ref{fig:communication} and we will introduce them in what follows. Though several variants are available in the literature, the three algorithms described next are baseline representations of the approaches studied and well suitable for our purposes.

\begin{figure}[ht]
    \includegraphics[width = \columnwidth]{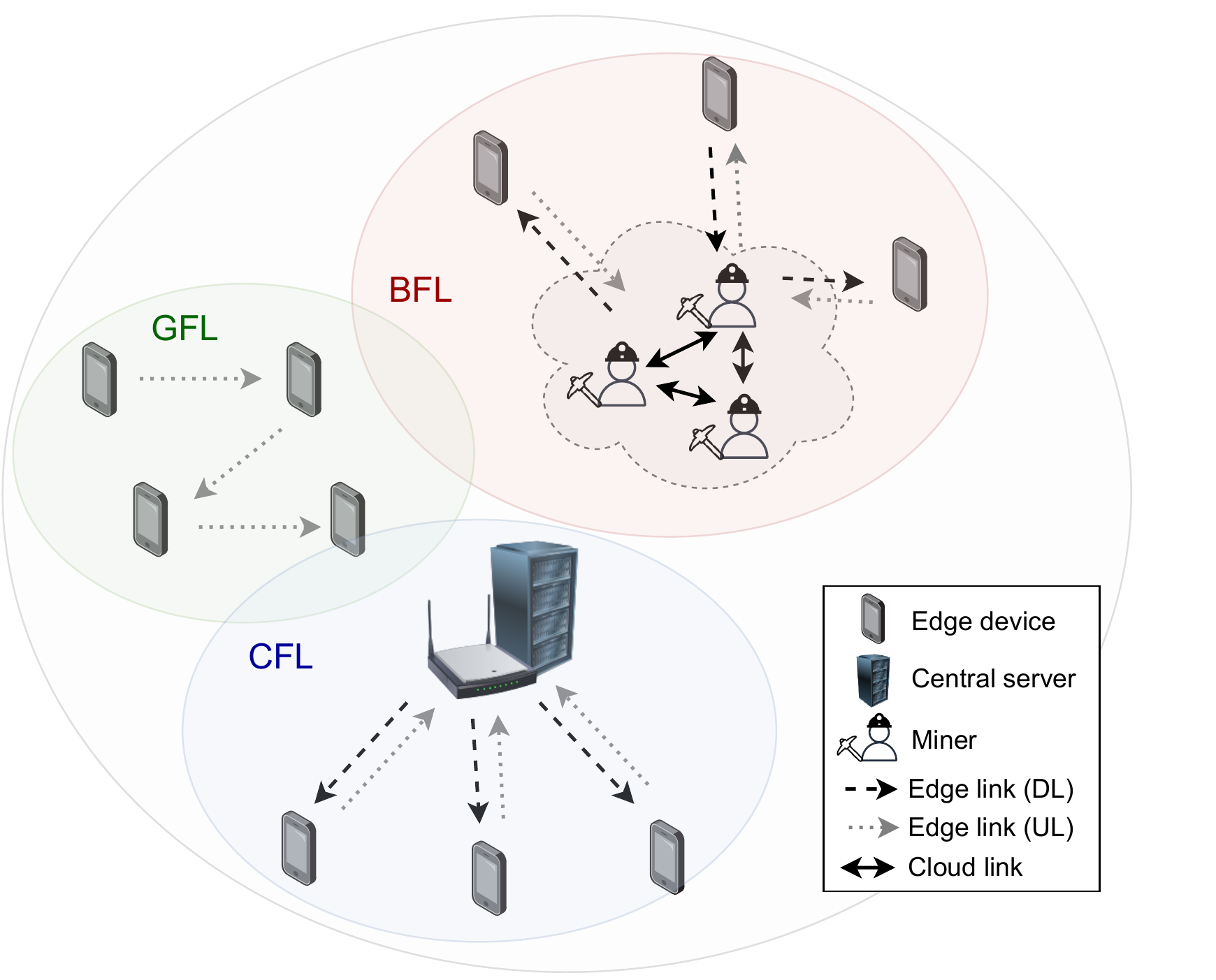}
    \caption{Overview of the different FL scenarios.}
    \label{fig:communication}
\end{figure}

\subsection{Centralized Federated Learning (CFL)}\label{sub:CFLearning}

At the beginning of a round $t$, a random subset of $m$ devices $\mathcal{S}^t \subseteq \mathcal{N}$ is selected, and the server sends the current global model to the parties. Each client makes $E$ epochs on the local dataset with a  mini-batch size of $B$, updates its local model $w_k^{t+1}$ and sends it to the server. The server aggregates the local updates and generates the new global model by computing the weighted average of the local updates as follows:
\begin{equation}
\label{eq:global_update}
    w^{t+1} = \sum_{k \in \mathcal{S}^t} \frac{|D_k|}{|D|}w_k^{t+1},
\end{equation}
where $|D| = \sum_{k \in \mathcal{S}^t} |D_k|$. The process is repeated until the model reaches convergence, e.g., the loss function does not improve across subsequent epochs or a specific number of training rounds have been executed. In this work, we consider the FedAvg algorithm~\cite{mcmahan2017communication} as a merging method to generate global model updates.

Algorithm~\ref{alg:FedAVG} describes the CFL with FedAvg mechanism. The procedure \textsc{Main} is executed by the server that coordinates the whole training process. Each client executes the procedure \textsc{ClientUpdate} and applies the stochastic gradient descent (SGD) algorithm on its local dataset with a learning rate $\eta$.
\begin{algorithm}
\caption{CFL}\label{alg:FedAVG}
    \begin{algorithmic}[1]
        \Procedure{Main}{}
        \State initialize $w_0$\;\label{CFL:startB1}
        \State$t \leftarrow 0$
        \While{convergence is not reached}
            \State {$\mathcal{S}^t \leftarrow$ random set of $m$ clients}
            \For {each client $k\in \mathcal{S}^t$ \textbf{in parallel}}
                \State {Download the global model $w^t$}
                \State {$w^{t+1}_k \leftarrow$ \textsc{ClientUpdate}($k$, $w^t$) }
                \State Send $w_k^{t+1}$ to the server\
            \EndFor\label{CFL:endB1}
            \State{$w^{t+1} \leftarrow \sum_{k \in \mathcal{S}^t} \frac{|D_k|}{|D|}w^{t+1}_k $}
            \State $t \leftarrow t + 1$ \label{CFL:endB2}
        \EndWhile 
    \EndProcedure
    \end{algorithmic}
    \begin{algorithmic}[1]
        \Procedure{ClientUpdate$(k, w)$}{}\Comment{Run on client $k$}
        \State $\mathcal{B} \leftarrow$ split the local dataset into batches of size $B$\;
        \For {$E$ local epochs}
            \For {batch $b \in \mathcal{B}$} \label{l:for_clientupdate}
                \State $w \leftarrow w - \eta\nabla \ell (w,b)$ 
            \EndFor
        \EndFor
        \State \Return $w$
        \EndProcedure
    \end{algorithmic}
\end{algorithm}

\subsection{Blockchain-enabled Federated Learning (BFL)}\label{sub:FLchain}

BFL is based on distributed ledger technology, which  collects data in form of transactions and organizes it in blocks. Indeed, a blockchain is a sequence of blocks chained one after the other through advanced cryptographic techniques. Each block contains the hash value of the previous one, leading to a tampered-proof sequence and providing properties that are essential to building trust in decentralized settings, such as transparency and immutability. In a blockchain, a set of participant nodes (miners) apply certain mining protocols and consensus mechanisms to append new blocks to the blockchain and agree on the status of the same. This procedure allows devices to write concurrently on a distributed database and guarantees that any malicious change on data would not be accepted by the majority, so that data in a blockchain is secured. 

When a blockchain is applied to a federated setting, the process is going as follows~\cite{wilhelmi2021blockchain}:
\begin{enumerate}
    \item Each device submits its local model updates in the form of transactions to the blockchain peer-to-peer (P2P) network of miners.
    \item The transactions are shared and verified by miners. \label{process:BC_start}
    \item Miners execute certain tasks to decide which node updates the chain. One of the most popular mining mechanisms, and studied in this paper, is Proof-of-Work (PoW)~\cite{nakamoto2008bitcoin}, whereby miners spend their computational power (denoted by $\lambda$) to solve computation-intensive mathematical puzzles.  
    \item As a result of the concurrent mining operation, a new block is created and propagated throughout the P2P blockchain network every $BI$ seconds (on average). The block size $S_B$ is selected such that can include a maximum of $m$ transactions, each one representing a local model submitted by a client.\label{process:BC_end}
    \item Clients download the latest block from its associated miner (as in~\cite{kim2019blockchained, nguyen2021federated}), which would allow performing on-device global model aggregation and local training.
\end{enumerate}

An important consequence of the blockchain decentralized consensus is forking. A fork occurs when two or more miners generate a valid block simultaneously (i.e., before the winning block succeeds to be propagated). The existence of forks can be seen as a waste of resources, as it may lead to extra computation and delay overheads~\cite{wilhelmi2021discrete}.
In this work, we consider the version of BFL reported in Algorithm~\ref{alg:FLchain}~\cite{wilhelmi2021blockchain}, which entails the participation of 
multi-access edge computing (MEC) servers and edge devices. Each client downloads the updates ${w^t_1 ... w^t_{m}} \in b^t$ contained in the latest block, computes the new global $w^t$, and trains it on its local dataset with the \textsc{ClientUpdate} procedure described in Section~\ref{sub:CFLearning}. The parameters of the new updated model $w^{t+1}_k$ are then submitted with the method \textsc{SubmitLocalUpdate}, where $S_{tr}$ is the transaction size. Once all the local updates are uploaded to the blockchain, a new block $b^{t+1}$ is mined with \textsc{MineBlock}, 
where the block generation rate, $\lambda=\frac{1}{BI}$, is derived from the total computational power of blockchain nodes. 
Finally, the new block is shared across all the blockchain nodes with the procedure \textsc{PropagateBLock}, which depends on the size of block $b^{t+1}$ (fixed to $S_B$). The process is repeated until convergence. 

\begin{algorithm}
    \caption{BFL}\label{alg:FLchain}
    \begin{algorithmic}[1]
        \Procedure{Main}{}
            \State $t \leftarrow 0$\;
            \State initialize $w_0$\;
            \While {convergence is not reached} 
                \State {$\mathcal{S}^t \leftarrow$ random set of $m$ clients}\; 
                \For {each client $k\in \mathcal{S}^t$ \textbf{in parallel}}
                    \State {Download the latest block, $b^t$}
                    \State {$w^{t} \leftarrow \sum_{j\in b^t} \frac{|D_j|}{|D|}w_j^{t}$}\label{FLchain:update}
                    \State {$w^{t+1}_k \leftarrow$ \textsc{ClientUpdate}($k$, $w^t$)}
                    \State {\textsc{SubmitLocalUpdate($S_{tr}$)}}
                \EndFor
                \State{$b^{t+1}\leftarrow$\textsc{MineBlock}($\lambda$)} \label{line:mine_b}
                \State{\textsc{PropagateBlock}($b^{t+1}$)} 
                \If {$b^{t+1}$ \textbf{is not valid}}
                    \State{Go to line~\ref{line:mine_b}}
                \EndIf
                \State{$t \leftarrow t+1$}\; \label{CFL:counterupdate}
            \EndWhile
        \EndProcedure
    \end{algorithmic}
\end{algorithm}

\subsection{Gossip Federated Learning (GFL)}\label{sub:GL}
GFL is an asynchronous protocol that trains a global model over decentralized data using a gossip communication algorithm~\cite{ormandi2013gossip, giaretta2019gossip}. 

We consider the general skeleton proposed in~\cite{giaretta2019gossip} and~\cite{miozzo2021distributed}. Overall, the participating clients start from a common initialization. The global model is then trained sequentially on local data and following a given path (e.g., random walk) of visiting clients. 

Algorithm~\ref{alg:gfl} describes the GFL procedure. At each round $t$, $m$ nodes are randomly selected and ordered in a sequence $\mathcal{S}^t = [k_1, ..., k_{m}]$. Every node $k_i \in \mathcal{S}^t$ receives the model $w_{k_{i-1}}^t$ from the previous node in the sequence and performs the steps below, also reported in Fig.~\ref{fig:GFL_node_update} for a better understanding:
\begin{enumerate}
    \item Run the procedure \textsc{Merge} to combine $w_{k_{i-1}}^t$ and the model saved in its local cache, $\rm{lastModel}_{k_i}$, i.e., the model from the previous round in which the node have been selected.
    \item Train the model generated in 1) on the local dataset with the procedure \textsc{ClientUpdate} described in Section~\ref{sub:CFLearning}.
    \item Update the local cache ($\rm{lastModel}_{k_i}$) with the model received from the previous node $w^t_{k_{i-1}}$.
    \item Share the model trained on the local dataset $w_{k_i}^t$ with the following node in the sequence.
\end{enumerate}
A round is completed when the model has visited all the nodes in the sequence. The algorithm stops when convergence is reached (after a given number of rounds).

\begin{algorithm}
    \caption{GFL}
    \label{alg:gfl}
    \begin{algorithmic}[1]
        \Procedure{Main}{}
            \State initialize $lastModel_k$ for each client $k$
            \State $t \leftarrow 0$
            \While {convergence is not reached}
                \State $\mathcal{S}^t \leftarrow$ random set of $m$ clients
                \State $[k_1, ..., k_{m}] \leftarrow \textsc{GetSequence}(\mathcal{S}^t)$
                \For {$i = 1, ..., m $}
                    \State $w^t_{k_{i-1}}$ \Comment{Model trained by the previous node in the sequence}
                    \State {$w^t_{k_i} \leftarrow \textsc{Merge($w^t_{k_{i-1}}, \rm{lastModel}_{k_i}$)}$} \label{l:merge_gfl}
                    \State $w^t_{k_i} \leftarrow \textsc{ClientUpdate}(k_i, w^t_{k_i})$
                    \State $\rm{lastModel}_{k_i}$ $\leftarrow$ $w^t_{k_{i-1}}$
                    \State Send model to the next client
                \EndFor
                \State $t\leftarrow t+1$
            \EndWhile
        \EndProcedure\\
        \Procedure{Merge$(w,w')$}{}
            \State $w \leftarrow \frac{w + w'}{2}$
            \State \Return $w$
        \EndProcedure
        \Procedure{GetSequence($S^t$)}{}
            \State $[k_1, \ldots, k_{m}] \overset{\text{i.i.d.}}{\sim} U(S^t)$
            \State \Return $[k_1, \ldots, k_{m}]$
        \EndProcedure
    \end{algorithmic}
\end{algorithm}

\begin{figure}[ht]
    \includegraphics[width = \columnwidth]{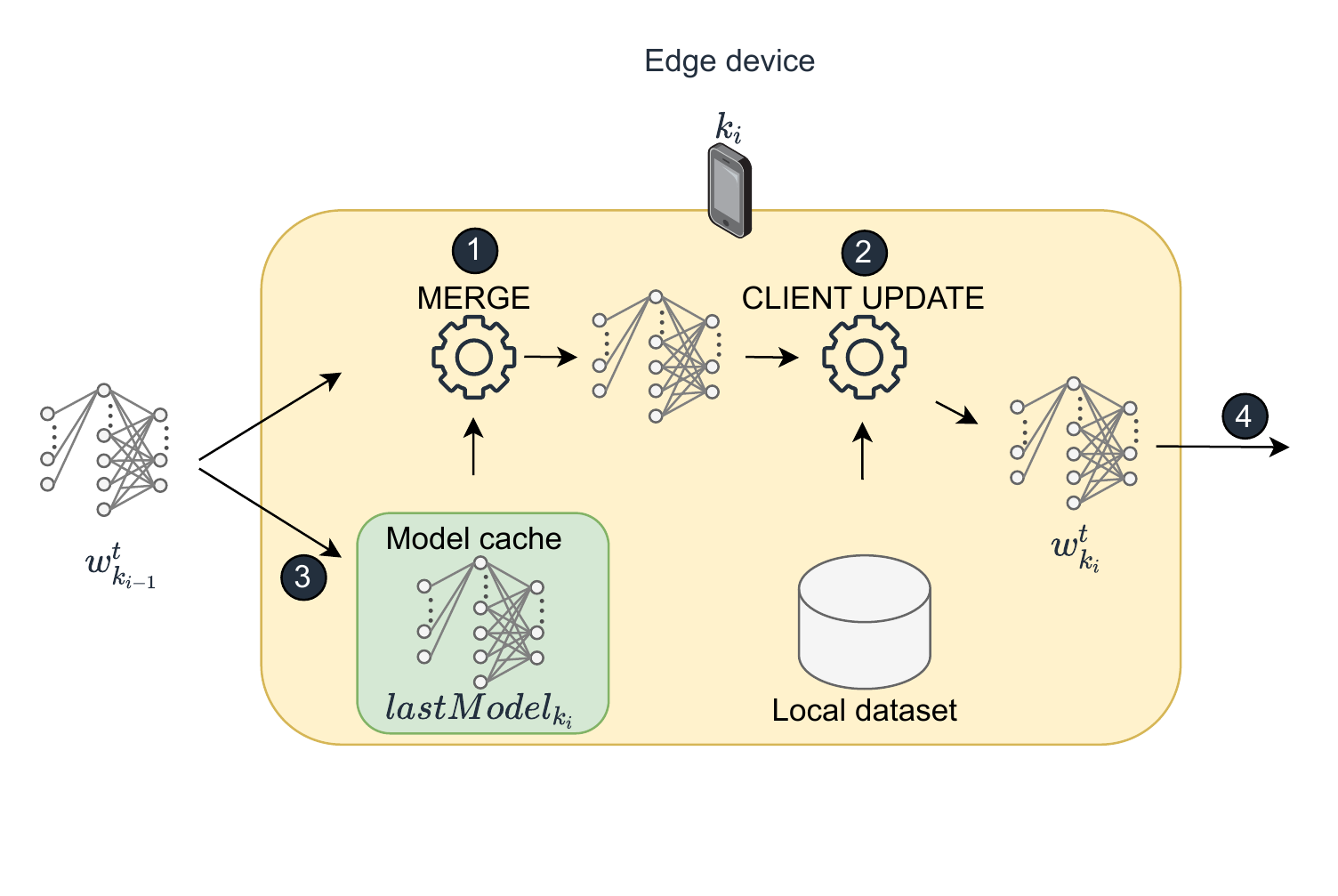}
    \caption{Overview of the operations executed by a node in the GFL algorithm.}
    \label{fig:GFL_node_update}
\end{figure}
\section{Computational and Communication Costs}
\label{sec:theoretical_analysis}
In this section, we introduce the mathematical statements for the calculation of the time complexity of the three federated algorithms considered in Section~\ref{sec:background}. 
We elaborate also on the data overhead due to the communication of the different model updates during the several rounds of the process of each implementation. Finally, we derive the equations for the calculation of the time to reach the convergence of the three analyzed federated approaches.
The results proposed hereafter are derived using the following assumptions:
\begin{enumerate}
    \item Scalar operations (sums and products) cost $O(1)$.
    \item The time complexity of the matrix multiplication is linear with the matrix size, i.e., $A \in \mathbb{R}^{i \times j}$ and $B\in \mathbb{R}^{j \times k}$, the cost of the product is $O(i \cdot j \cdot k)$.
    \item For a single input pair $(x_i, y_i)$, the time complexity required to compute $\nabla \ell$ is linear with the number of model's weights, $O(|w|)$.
    \item During the mining process, with the PoW, a miner computes the nonce of a block using brute force until finding a hash value lower or equal to a certain threshold~\cite{zheng2017overview}, referred to as the mining difficulty. Assume that the hash value has $b$ bits, and that its solution should be smaller than $2^{b-l}$ bits (being $l$ a value determined by the mining difficulty), if the miner samples the nonce values at random, the probability of a valid value is $2^{-l}$. Henceforth, $2^{l}$ sampling operations are required for mining a block. The time complexity is $O(2^l)$.
    \item The set of nodes that have a local copy of the blockchain is $\mathcal{N}_{B} = \{1, ... , N_B\}$, without loss of generality, is assumed to be $\mathcal{N} \cap \mathcal{N}_{B} = \varnothing$.
    \item We assume that convergence of the FL training procedure is reached after $R$ rounds. 
\end{enumerate}

\begin{theorem}
The time complexity of CFL is: 
    \begin{equation}
        O(RmE|D_{\max}||w|),
    \label{eq:CFL_cost}
    \end{equation}
where $D_{\max} = \max_{k \in \mathcal{N}} |D_k|$. 
The communication overhead is given by
\begin{equation}
    2Rm|w|
    \label{eq:CFL_comm}
\end{equation} 
\end{theorem}
\begin{proof}
See Appendix~\ref{sub:cfl_proof}.
\end{proof}

\begin{theorem}
The time complexity of BFL is:
    \begin{equation}
         O(R (|w|m^2 + E|D_{\max}||w|m + 2^l + m|w|N_B)),
         \label{eq:BFL_cost}
    \end{equation}
where $N_B$ is the number of nodes that have a local copy of the blockchain and $l$ is related to the PoW difficulty (see Assumption 4). Its communication overhead is 
    \begin{equation}
        R\left(|w|m^2 + |w|m + m|w|N_B\right)
        \label{eq:BFL_comm}
    \end{equation}
\end{theorem}
\begin{proof}
    See Appendix~\ref{sub:bfl_proof}.
\end{proof}

\begin{theorem}
The time complexity of GFL is 
    \begin{equation}
         O(RmE|D_{\max}||w|)
    \label{eq:GFL_cost}
    \end{equation}
and its communication overhead is 
\begin{equation}
    Rm|w|
\label{eq:GFL_comm}
\end{equation}
\end{theorem}
\begin{proof}
    See Appendix~\ref{sub:gfl_proof}.
\end{proof}

\begin{table*}
\caption{Computational complexity and communication overhead for CFL, BFL and GFL.}
\centering
    \begin{tabular}{c|cc}
        \toprule
        \textbf{Algorithm} & \textbf{Time complexity} & \textbf{Communication Overhead} \\
        \midrule
        CFL & $O(RmE|D_{\max}||w|)$  & $2Rm|w|$ \\
        BFL &  $O(R (|w|m^2 + E|D_{\max}||w|m + 2^l + m|w|N_B))$ & $R(|w|m^2 + |w|m + m|w|N_B)$\\
        GFL & $O(RmE|D_{\max}||w|)$ & $Rm|w|$\\
        \bottomrule
    \end{tabular}
    \label{tab:th_comparison}
\end{table*}
The three algorithms have a time complexity that depends on the dataset size $|D_{\max}|$. 
Additionally, the time complexity of BFL~\eqref{eq:BFL_cost} is also a function of the blockchain parameters $N_B$ and $l$. In particular, the dominant term in~\eqref{eq:BFL_cost} is $R2^l$. 
Hence, the time complexity of BFL is exponential in the PoW difficulty $l$, while for CFL and GFL is polynomial in $RmE|D_{\max}||w|$.
Table~\ref{tab:th_comparison} summarizes the different results obtained for time complexity and communication overhead.

To finalize our analysis, we compute the total execution time of each algorithm till convergence (convergence time) as a function of the delay introduced by the communication rounds and the computational operations. 
To do that, we characterize the links whereby the different types of nodes exchange information (e.g., local model updates, blocks), having in mind the topology introduced in Fig.~\ref{fig:communication}.
We classify two different types of connections: cloud (solid arrows) and edge (dotted and dashed arrows). Cloud connections (assumed to be wired) are used by miners in the blockchain; instead, edge connections (assumed to be wireless) are used by edge nodes to upload/download models.
Given its popularity and easiness of deployment, we adopt IEEE 802.11ax links for edge connections~\cite{bellalta2016ieee}. Since edge devices are often energy-constrained, we consider different values of transmission power for the edge connections. The central server and blockchain node use a transmission power of $P_{\text{TX}}^{c}$, instead, edge devices use $P_{\text{TX}}^{e}$, with $P_{\text{TX}}^{e} \leq P_{\text{TX}}^{c}$. The wired connection has a capacity $C_{P2P}$.
Additionally, we identify three main types of computational operations during the federated learning processes: local model training, model parameters exchange, and blockchain data sharing. Based on this, we can compute the convergence time of CFL, BFL, and GFL as follows:
\begin{equation}
    T_{\rm CFL} = T_{\rm train} + T_{\rm{Tx}}^{e} + T_{\rm{Tx}}^c,
    \label{eq:Tcfl}
\end{equation}
\begin{equation}
    T_{\rm BFL} = T_{\rm BC} + T_{\rm train} + T_{\rm{Tx}}^{e}+ T_{\rm{Tx}}^{c},
    \label{eq:Tbfl}
\end{equation} 
\begin{equation}
    T_{\rm GFL} = T_{\rm train} + T_{\rm{Tx}}^{e},
    \label{eq:Tgfl}
\end{equation}
where $T_{\rm train}$ is the total amount of time spent for training the ML model locally, $T_{\rm{Tx}}^{c/e}$ is the total transmission time of the central server/blockchain nodes $(c)$ or the edge devices $(e)$, and computed according to the model detailed in Appendix~\ref{sec:edge_conn_model}. $T_{\rm BC}$ is the total delay introduced by blockchain and described in steps~\ref{process:BC_start}-\ref{process:BC_end} of the process in Section~\ref{sub:FLchain}. 
\section{Energy Footprint}
\label{sec:energy_communication_model}

In this section, we define the models used to characterize the energy consumption that results from the FL operations. Driven by \eqref{eq:Tcfl},  \eqref{eq:Tbfl} and  \eqref{eq:Tgfl}, the total amount of energy consumed in each scenario is:
\begin{equation}
    \mathcal{E}_{\rm CFL} = \mathcal{E}_{\rm train} + \mathcal{E}_{\rm{Tx}}^{e} + \mathcal{E}_{\rm{Tx}}^{c},
    \label{eq:Ecfl}
\end{equation}
\begin{equation}
    \mathcal{E}_{\rm BFL} = \mathcal{E}_{\rm BC} + \mathcal{E}_{\rm train} + \mathcal{E}_{\rm{Tx}}^{e} + \mathcal{E}_{\rm{Tx}}^{c},
    \label{eq:Ebfl}
\end{equation}
\begin{equation}
    \mathcal{E}_{\rm GFL} = \mathcal{E}_{\rm train}+ \mathcal{E}_{\rm{Tx}}^{e},
    \label{eq:Egfl}
\end{equation}
where $\mathcal{E}_{\rm train}$ is the energy consumed by all the nodes during the local training,
and $\mathcal{E}_{\rm{Tx}}^{c/e}$ the energy required to transmit the model through the IEEE 802.11ax wireless links during the whole algorithm execution, from either a central server/blockchain node ($c$) or an edge device ($e$). 
$\mathcal{E}_{train}$ is calculated as:
\begin{equation}
    \mathcal{E}_{\rm train} = \sum_{r = 0 }^{R-1} \sum_{i \in \mathcal{S}^r} P_{{\rm CPU}_i}^{r}\Delta_i^r,
\end{equation}
where $P_{{\rm CPU}_i}^{r}$ is the average power consumed by the CPU and DRAM during a round $r$, and $\Delta_i^r$ is the duration of the operation for the $i-th$ client. 
As described in Section~\ref{sec:theoretical_analysis}, we may have two types of communication links: cloud and edge. Considering that cloud links are wired, we assume that their energy consumption is negligible.
Instead, we compute the energy consumption of the edge connections according to the following equation:
\begin{equation}
    \mathcal{E}_{\rm{Tx}}^{c/e} = T_{\rm{Tx}}^{c/e}P_{\rm{Tx}}^{c/e},
\end{equation}
where $T_{\rm{Tx}}^{c/e}$ and $P_{\rm{Tx}}^{c/e}$ are the transmission time and power of a central server/miner $(c)$ or an edge device $(e)$.
The additional term $\mathcal{E}_{\rm BC}$ for BFL is associated to mining operations of the PoW. We measure that consumed energy based on the model proposed in~\cite{lasla2020green} and according to the following equation:
\begin{equation}
    \mathcal{E}_{\rm BC} = P_{h}\frac{1}{\lambda}N_{\rm chain},
    \label{eq:EPow}
\end{equation}
where $P_{h}$ is the total hashing power of the network, $\lambda$ is the block generation rate, and $N_{\rm chain}$ is the number of blocks on the main chain.

\section{Performance Evaluation}
\label{sec:exp_settings}
In this section, we first describe the experimental settings adopted to compare the three federated approaches and then, we present numerical results. 

\subsection{Simulation Setup}

We use the \textit{Extended MNIST (EMNIST)} dataset available on \textit{Tensorflow Federated (TFF)} library~\cite{emnist}. The input features are black and white images that represent handwritten digits in $[0,9]$, coded in a matrix of $28 \times 28$ pixels.
Considering only digits, it contains $341\,873$ training examples and $40\,832$ test samples, both divided across $3\,383$ users. The training and the test sets share the same users' list so that each user has at least one sample.  
In the EMNIST dataset, all the clients have a rich number of samples for all the classes, thus data distribution across them can be considered as IID. To evaluate the targeted federated mechanisms in more challenging settings, we create a new version of the EMNIST dataset, called EMNISTp, by randomly restricting each client dataset to $3$ classes only. Fig.~\ref{fig:dataset_distribution} shows the available samples of the first $4$ clients, for both versions of the dataset.

\begin{figure}[ht!]
    \centering
    \includegraphics[width=\columnwidth]{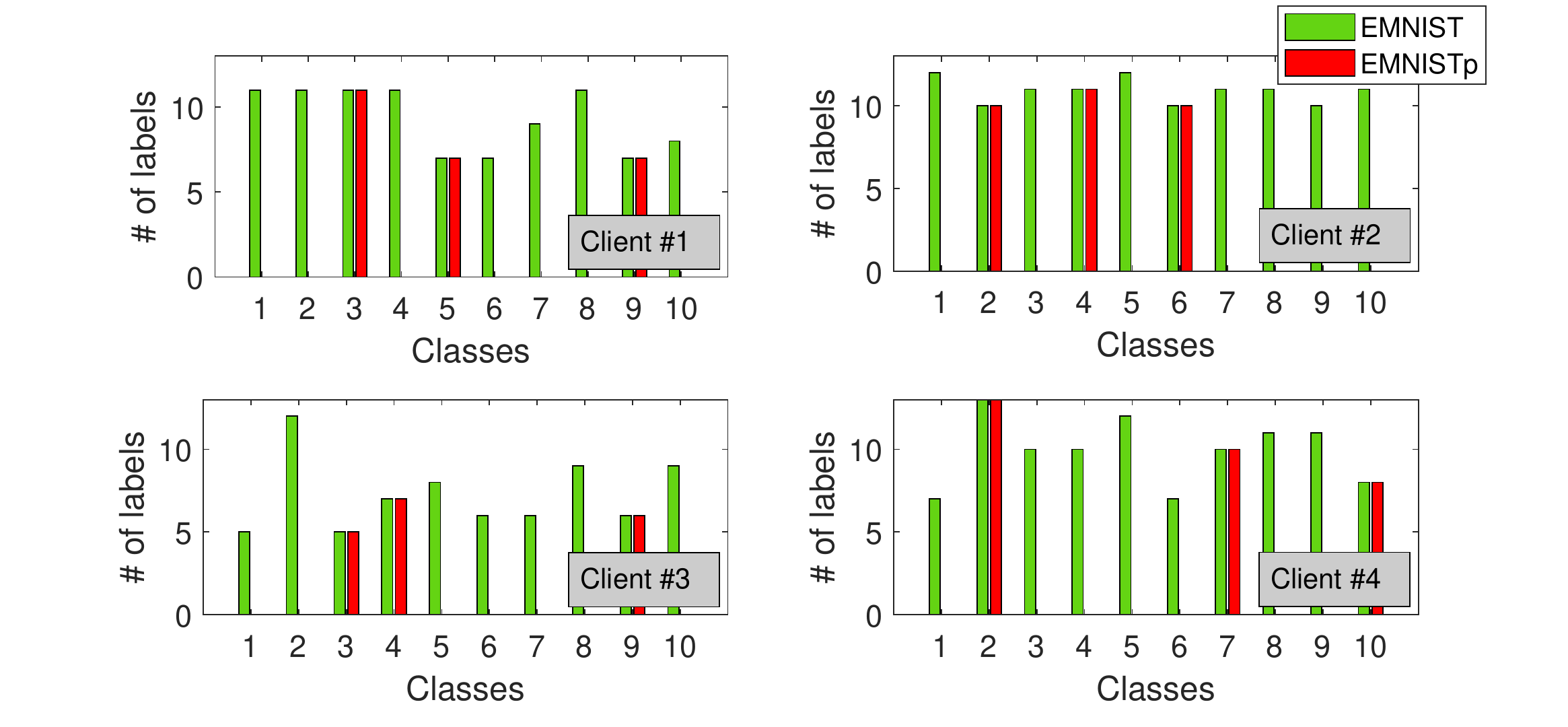}
    \caption{Distribution of samples across the four first clients for both EMNIST and EMNISTp federated datasets.}
    \label{fig:dataset_distribution}
\end{figure}

To correctly classify these samples we choose two models proposed in~\cite{mcmahan2017communication}.
The first one is a feed-forward neural network (FFNN) with an input layer of $784$ neurons, two hidden layers of $200$ neurons activated with the rectified linear unit (ReLU) function, and an output layer of $10$ neurons with \emph{Softmax} activation function. The number of trainable parameters for the FFNN $(|w^\prime|)$ is $\num{199210}$. Assuming that each parameter requires $4$ bytes in memory, i.e., size of a \emph{float32} variable, the total amount of space required ($S_{w^{\prime}}$) is $796.84$ KB. The second one is a convolutional neural network (CNN) with the following structure:
\begin{enumerate}
    \item A $5 \times 5$ convolutional layer with $32$ channels and ReLU activation function.
    \item A $2 \times 2$ max pooling layer.
    \item A $5 \times 5$ convolutional layer with $64$ channels and ReLU activation function.
    \item A $2 \times 2$ max pooling layer.
    \item A fully connected layer with $512$ units and ReLu activation.
    \item An output layer with $10$ neurons and Softmax activation function.
\end{enumerate}
In total, the number of trainable parameters for the CNN $(|w''|)$ is $\num{582026}$ and the size in memory $(S_{w''})$ is $2.33$ MB.
We opted for a FFNN to reproduce a realistic scenario where edge devices might not have enough computational power to train more sophisticated deep learning mechanisms, like NNs based on convolutional architectures. Despite of its simplicity, the selected FFNN model accurately classifies the digits of the EMNIST dataset, as shown next. Then, we use also a CNN to evaluate the multiple algorithms' performance using a more complex model (details in Section \ref{sub:results_cnn}).

The three FL algorithms are implemented with \textit{Tensorflow (TF)}~\cite{tensorflow2015-whitepaper}, \textit{Tensorflow Federated (TFF)}~\cite{tff} and \textit{Keras}~\cite{noauthor_keras:_nodate} libraries. 
We have extended the Bitcoin model provided by BlockSim~\cite{10.3389/fbloc.2020.00028} to simulate the blockchain behavior. BlockSim is an event-based simulator that characterizes the operations carried out to store data in a blockchain, from the submission of transactions to mining blocks and reaching consensus in a decentralized manner. Accordingly, BlockSim allows simulating the delays added by the blockchain in a BFL application, i.e., the $T_{BC}$ parameter defined in Section~\ref{sec:theoretical_analysis}.

We create a validation set by choosing a subset of $200$ clients from the test set. 
Following the TFF documentation~\cite{tff_tutorial}, the training accuracy is computed at the beginning of each round; instead the validation accuracy is calculated at the end. For this reason may happen that the validation accuracy is higher than the training one.
At the end of each simulation, we evaluate the performance on the test set.

As for $\mathcal{E}_{\rm train}$ and $T_{\rm train}$, we have used Carbontracker~\cite{anthony2020carbontracker}, a Python library that periodically samples the hardware energy consumption and measures the execution time. 
Moreover, $P_{\rm Tx}^c$ is set to $20\text{ dBm}$ and $P_{\rm Tx}^e=9\text{ dBm}$.
Table~\ref{tab:sim_param} reports all the other parameters used in our simulations.
We note here that we have used the same FL parameters for a fair comparison, being the number of rounds $R$ of CFL and GFL equivalent to the main chain's length ($N_{\rm chain}$) in BFL. In such a way, we guarantee that the number of global rounds of each learning algorithm is the same. We run the experiments with the following hardware configurations: Intel i5-6600 with $8\text{GB}$ of RAM (HW1) and two Intel Xeon 6230 with $188\text{GB}$ of RAM (HW2). 

\begin{table}[ht]
    \caption{Simulation parameters.}
    \label{tab:sim_param}
    \centering
    \resizebox{\columnwidth}{!}{  
    \begin{tabular}{c|ccc}
        \toprule
           & \textbf{Parameter} & \textbf{Description} & \textbf{Value} \\
         \midrule
          \multirow{11}{*}{\rotatebox[origin=c]{90}{Fed. Learning }} 
          & $|w'|$ & Number of FFNN model parameters & $\num{199210}$ \\
          & $|w''|$ & Number of CNN model parameters & $\num{582026}$ \\
          & $S_{w^'}$ & FFNN model parameters size & $796.84$ KB \\
          & $S_{w''}$ & CNN model parameters size & $2.33$ MB \\
          & $\eta$ & Learning rate & $0.2$ \\
          & $N$ & Number of total clients & $3382$ \\
          & $E$ & Local epochs number & $5$ \\
          & $R$ & Number of rounds & $200$ \\
          & $m$ & Number of clients for each round & $200$ \\
          & $B$ & Batch size & $20$ \\
          & $\ell_i$ & Local loss function & Sparse Cat. Crossentropy\\
          \midrule
          \multirow{11}{*}{\rotatebox[origin=c]{90}{Blockchain}}
          & $N_{\rm chain}$ & Number of blocks in the main chain & 200 \\
          & $BI$ & Block interval & $15$ s\\
          & $N_{B}$ & Number of blockchain nodes & $200$\\
          & $N_{m}$ & Number of miners & $10$\\
          & $C_{\rm P2P}$ & Capacity of P2P links & $100$ Mbps\\
          & $S_H$ & Block header size & $25$ KB\\ 
          & $S'_B$ & Block size with FFNN & $160.368$ MB \\
          & $S''_B$ & Block size with CNN & $467$ MB \\
          & $S'_{\rm tr}$ & Transaction size with FFNN & $796.84$ KB\\
          & $S''_{\rm tr}$ & Transaction size with CNN & $2.33$ MB\\
          & $P_h$ & Total hashing power & $1350$ W\\
         \midrule
         \multirow{17}{*}{\rotatebox[origin=c]{90}{Communication (IEEE 802.11ax)}} 
         & $P_{\rm{Tx}}^{e}$ & Tx power for edge devices & $9$ dBm \\
         & $P_{\rm{tx}}^{c}$ & Tx power for a central server & $20$ dBm\\
         & $\sigma_{\text{leg}}$ & Legacy OFDM symbol duration & $4$ \textmu s \\
 		 & $N_{\rm sc}$ & Number of subcarriers ($20$ MHz) & $234$ \\
		 & $N_{\rm ss}$ & Number of spatial streams & $1$ \\
         & $T_{\rm{e}}$ & Empty slot duration & $9$ \textmu s \\
		 & $T_{\rm{SIFS}}$ & SIFS duration & $16$ \textmu s \\ 
		 & $T_{\rm{DIFS}}$ & DIFS duration & $34$ \textmu s \\ 
		 & $T_{\rm{PHY}}$ & Preamble duration & $20$ \textmu s \\
		 & $T_{\rm{HE-SU}}$ & HE single-user field duration & $100$ \textmu s \\
		 & $L_{s}$ & Size OFDM symbol & $24$ bits \\ 
		 & $L_{\rm{RTS}}$ & Length of an RTS packet & $160$ bits \\ 
		 & $L_{\rm{CTS}}$ & Length of a CTS packet & $112$ bits \\
 		 & $L_{\rm{ACK}}$ & Length of an ACK packet & $240$ bits \\ 
		 & $L_{\rm{SF}}$ & Length of service field & $16$ bits \\ 
		 & $L_{\rm{MAC}}$ & Length of MAC header & $320$ bits \\
		 & $\text{CW}$ & Contention window (fixed) & $15$ \\
        \bottomrule
    \end{tabular}}

\end{table}

\subsection{Result Analysis}

Table~\ref{tab:resultsEMNIST} reports the accuracy of each algorithm implementation with the FFNN model on the two considered datasets executed on HW1. CFL and BFL achieve the best accuracy (both close to $0.9$), instead, GFL presents lower values. This result validates the claim that, under similar setups as in our simulations (i.e., each block contains $m$ local updates organized in transactions), the central parameter server of CFL can be replaced by a blockchain network, properly dimensioned, without compromising the learning accuracy.
On the other hand, GFL achieves a validation accuracy around $0.5$ after $200$ rounds. We justify this behavior by noticing that, before the \textsc{ClientUpdate} procedure, the model received from the previous node in the sequence is merged with that in the previous round. For the earliest training rounds, there is a high probability that the merging procedure is with a fresh model that has never been trained, hence disrupting the knowledge from the previous clients. We analyze more in-depth this phenomenon in Section~\ref{sub:GossipOpen}.

\begin{table*}[ht]
\caption{FFNN simulation results on EMNIST (EMNISTp) datasets.}
\label{tab:resultsEMNIST}
\centering
\scriptsize
    \begin{tabular}{cccccccc}
    \toprule
        & \textbf{Acc. Training} & \textbf{Acc. Validation} & \textbf{Acc. Test}   & \textbf{Conv. Time (s)}            & \textbf{Comp. Energy (\%)} & \textbf{Tot. Energy (Wh)} & \textbf{Comm. Overhead (GB)} \\ \midrule
        CFL         & $0.9$ $(0.76)$    & $0.87$ $(0.77)$     & $0.86$ $(0.76)$ & $46458.56$ $(45571.86)$ & $98.91$ $(98.61)$     & $21.84$ $(17.2)$ & $63.75$ \\
        BFL     & $0.88$ $(0.74)$   & $0.87$ $(0.78)$     & $0.86$ $(0.77)$ & $51036.87$ $(50077.75)$ & $99.98$ $(99.98)$       & $1147.21$ $(1142.65)$ & 12781.31\\
        GFL          & $0.44$ $(0.36)$   & $0.42$ $(0.12)$     & $0.41$ $(0.11)$ & $38201.67$ $(36821.67)$ & $99.57$ $(99.14)$     & $17.83$ $(8.98)$ & $31.87$\\ 
        \bottomrule
    \end{tabular}
    
\end{table*}

Table~\ref{tab:resultsEMNIST} also details the convergence time
of each algorithm, the percentage of energy consumed in the computations (as a percentage of the total energy consumed), the total amount of energy needed and the communication overhead (i.e., data to be shared during the rounds of the algorithms). 
The fastest and the most energy-efficient algorithm is GFL: it is able to save the 18\% of training time, the 18\% of energy, and the 51\% of data to be shared with respect to the CFL solution. However, GFL main drawback resides in the poor accuracy achieved, as stated above. BFL is the slowest and the most energy-hungry federated implementation, mainly due to the overhead introduced by the blockchain network to secure data in a decentralized way.
Additionally, it is to be noted that computation is the most energy-consuming task for CFL, BFL, and GFL. For BFL, the mining process drains $1125$ Wh, i.e., $98\%$ of the total energy.
We highlight here that our comparison may be unfair in this respect, since both CFL and GFL are not including any security mechanism. However, we believe that it is worth to include BFL in our analysis on distributed versus centralized federated learning because, from our results, it emerges that the secure and decentralized method introduced by the blockchain network, despite increasing algorithm costs, does not jeopardize its accuracy compared to its centralized counterpart CFL.
Finally, GFL is the implementation that requires the lowest communication overhead. To be more precise, in this case, we need to include an extra cost to share the global model across the nodes at the end of the last round (not considered in the table), which is approximately 0.16 GB (the cost of one extra round)

\section{Open Issues and Research Directions}
\label{sec:open-aspects}

\subsection{Open Aspects of GFL}\label{sub:GossipOpen}
\begin{table*}[ht]
\caption{FFNN simulation results of GFL on EMNIST (EMNISTp) datasets with higher number of rounds and local computations.}
\label{tab:resultsGL}
\centering
\scriptsize
    \begin{tabular}{cccccccc}
\toprule
\textbf{R} & \textbf{E} & \textbf{Acc. Training} & \textbf{Acc. Validation} & \textbf{Acc. Test} & \textbf{Conv. Time (s)}     & \textbf{Comp. Energy (\%)} & \textbf{Tot. Energy (Wh)} \\ \midrule
$200$        & $5$          & $0.44$ $(0.36)$            & $0.42$ $(0.12)$              & $0.41$ $(0.11)$        & $36401.67$ $(35021.67)$   & $99.59$ $(99.19)$            & $17.83$ $(8.98)$           \\
$400$        & $5$          & $0.55$ $(0.41)$            & $0.51$ $(0.16)$              & $0.5$ $(0.15)$         & $72791.81$ $(70029.31)$   & $99.58$ $(99.18)$              & $35.08$ $(17.91)$          \\
$800$        & $5$          & $0.85$ $(0.64)$            & $0.67$ $(0.19)$              & $0.66$ $(0.17)$        & $145651.71$ $(140069.23)$ & $99.59$ $(99.18)$              & $70.95$ $(35.8)$           \\
$200$      & $10$         & $0.57$ $(0.58)$            & $0.4$ $(0.09)$               & $0.41$ $(0.1)$         & $37377.47$ $(35448.58)$   & $99.7$ $(99.38)$              & $24.55$ $(11.77)$            \\
$400$        & $10$         & $0.66$ $(0.37)$            & $0.47$ $(0.17)$              & $0.48$ $(0.16)$        & $74989.72$ $(70857.46)$   & $99.7$ $(99.37)$              & $49.17$ $(23.23)$          \\
$800$        & $10$         & $0.94$ $(0.58)$          & $0.67$ $(0.3)$               & $0.67$ $(0.29)$        & $149448.14$ $(141730.77)$ & $99.7$ $(99.38)$              & $98.37$ $(46.87)$\\
\bottomrule
\end{tabular}

\end{table*}

As described before, GFL is not able to achieve the same accuracy level as CFL and BFL. We identify two possible reasons for this behavior:
\begin{enumerate}
    \item The number of rounds is not enough to converge: the number of visited nodes might not be sufficient to hit an acceptable accuracy.
    \item The merge step negatively impacts the performance of the learning algorithm: the model received in the previous round and stored in the local cache slows down the learning process. 
\end{enumerate}  

To verify the first hypothesis, we execute GFL algorithm, with the FFNN model on HW1, changing the number of rounds ($R = \{200,400,800\} $) and varying the number of local computations ($E = \{5,10\}$). Table~\ref{tab:resultsGL} shows the results obtained. Considering the EMNIST dataset, the best results are achieved with $R=800$ and $E=\{5,10\}$, i.e., a higher test accuracy of $0.66$, but still lower than CFL and BFL. Moreover, the model is overfitting with $R=800$ rounds; hence, a regularization method would be needed, when increasing the number of rounds.
On the EMNISTp dataset, the accuracy is even lower for each combination of the hyperparameters tested. 

To verify the second hypothesis, we run GFL algorithm without the merge step (\emph{GFL-NM}). The pseudocode of this algorithm is the same in Algorithm~\ref{alg:gfl} but replacing the old Line~\ref{l:merge_gfl} with the new command $w_{k_i}^{t} \leftarrow w_{k_{i-1}}^t$. Thus, in GFL-NM, given a sequence of clients $\mathcal{S}_t$ the model is trained incrementally on the client's datasets.
GFL-NM achieves a training accuracy of $\sim1.0\text{ }(0.94)$, a validation accuracy of $0.94 \text{ }(0.78)$ and a test accuracy of $0.93 \text{ }(0.78)$ on the EMNIST (EMNISTp) dataset (see Fig.~\ref{fig:accuracy_GFL_noM}), higher than CFL and BFL. 
These results suggest that the \textsc{Merge} step compromises the training performance. In fact, at the beginning of the learning process, there is a high probability that a model visits a node that has never been visited before and with $\rm lastModel$ storing initialization values. In this case, the received model is merged with a model that has never been trained before, as shown in Algorithm~\ref{alg:gfl}, which negatively impacts the resulting merged weights.
Figure~\ref{fig:GFL_vs_GFLNM} shows the comparison between the learning curves of GFL and GFL-NM.
\begin{figure}[ht]
    \begin{subfigure}[b]{0.45\columnwidth}
        \centering
        \includegraphics[width=\columnwidth]{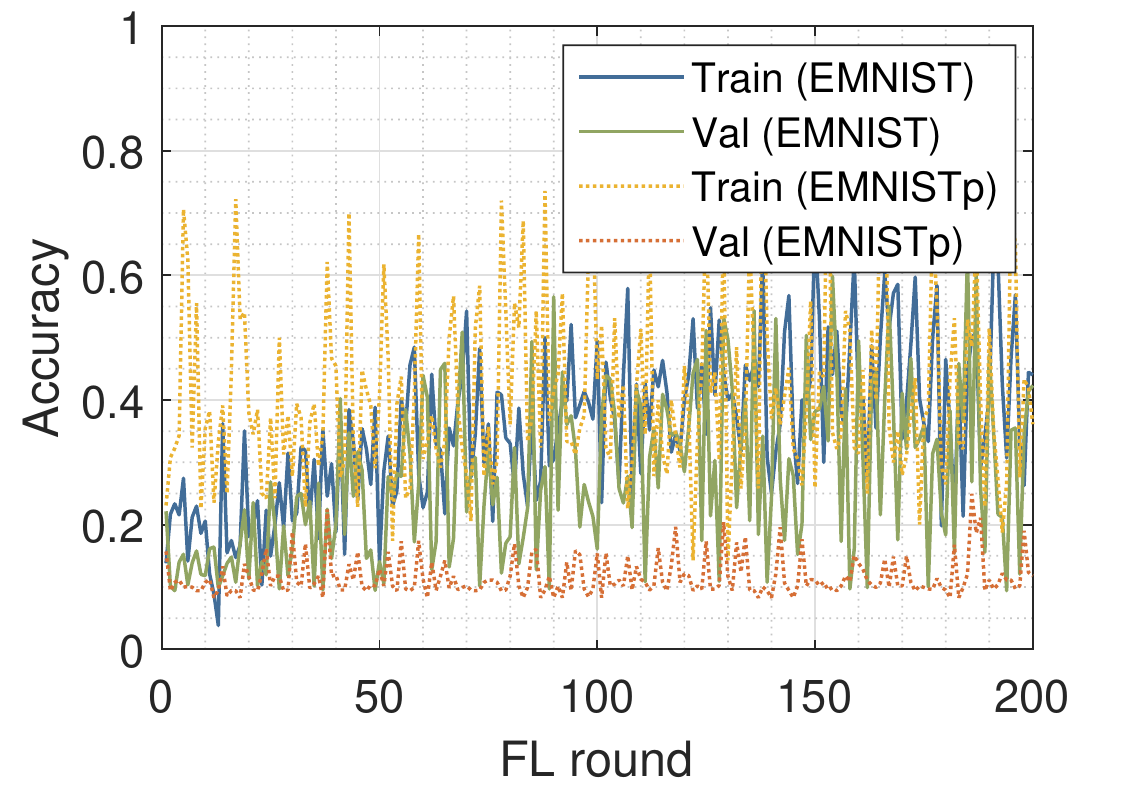}
        \caption{GFL}
        \label{fig:accuracy_GFL}
    \end{subfigure}
    \hfill
    \begin{subfigure}[b]{0.45\columnwidth}
        \centering
        \includegraphics[width=\columnwidth]{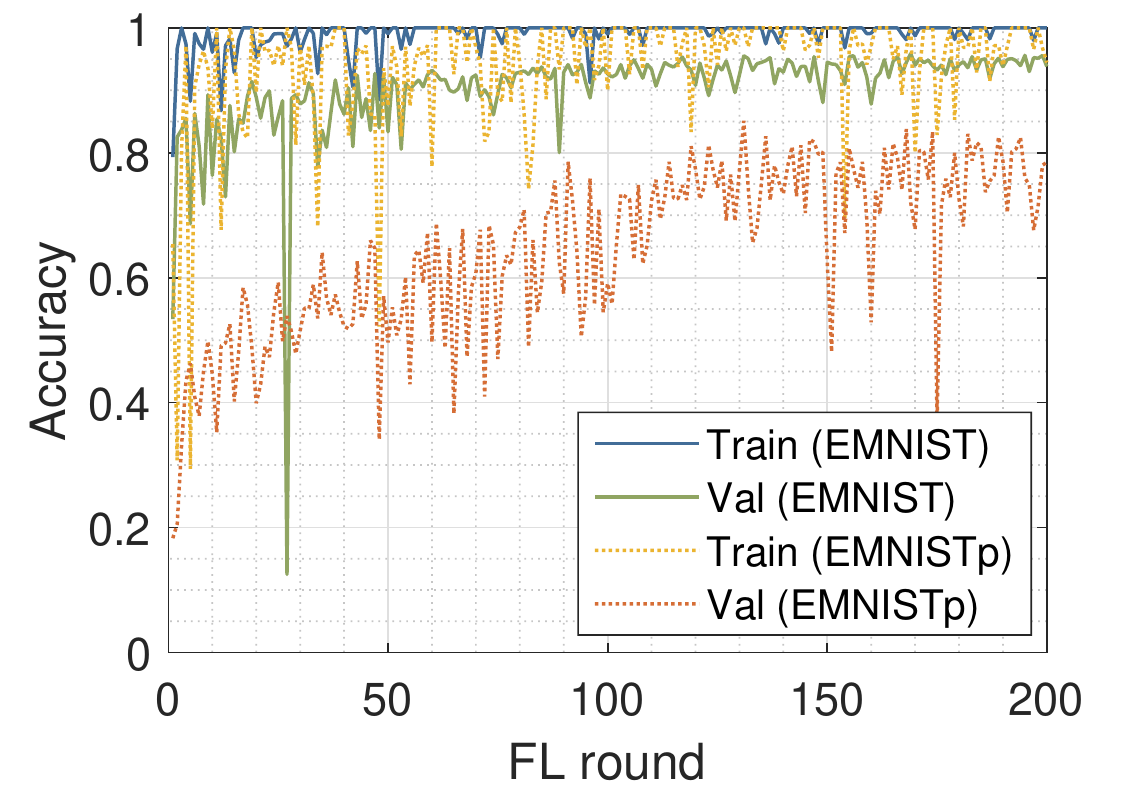}
        \caption{GFL-NM}
        \label{fig:accuracy_GFL_noM}
    \end{subfigure}
    \caption{Training and validation accuracy on EMNIST and EMNISTp}
    \label{fig:GFL_vs_GFLNM}
\end{figure}

In conclusion, we have seen that both 1) and 2) influence the achieved accuracy. Moreover, GFL-NM solves the accuracy problem of standard GFL and reaches the best performance from all the metrics point of view. 
In our opinion, and encouraged by our results, the investigation of new methods for merging the model updates from the distributed sources to achieve faster and higher accuracy is an interesting and open research line.
To the best of our knowledge, there are still very few works that go in this direction in the literature.
In \cite{miozzo2021distributed}, the authors implement an incremental version of GFL with a single round on the edge devices using the entire local dataset, and, hence, without requiring any merge step. Similarly, \cite{huang2022continual} proposes an iterative continual learning algorithm, where a model is trained incrementally on the local datasets without applying any merge operation.

\subsection{Open Aspects of BFL}

Blockchain technology, while enabling a reliable and secure FL operation, entails very high overheads in terms of time and energy for the sake of keeping decentralization. The performance of a blockchain, typically measured in transactions per second (tps), together with the granted degree of security, strongly depends on the nature of the blockchain (e.g., degree of visibility, type of consensus, mining protocol), its configurable parameters (e.g., block interval, difficulty), and the size of the P2P network maintaining it. Furthermore, the necessary energy to maintain a blockchain is correlated to its performance in tps and security, thus leading to the well-known performance, security, and energy trilemma.

To showcase the effect of using different types of blockchain networks, Fig.~\ref{fig:delay_blocks_fork_prob} shows the total delay incurred by the blockchain to the FL operation to generate up to 200 blocks under different blockchain configurations. Notice that, in the proposed setting, each block is equivalent to an FL round. In particular, we vary the total number of miners ($N_m=\{1, 10, 100\}$) and the block interval ($BI=\{5,15,600\}$ s), which affect the time required to achieve consensus. 

First, a higher number of miners leads to a higher fork probability, provided that more nodes need to agree on the same status of the ledger. By contrast, a higher block interval allows mitigating the effect of forks, since the probability that two miners mine a block simultaneously is lower~\cite{shahsavari2019theoretical}.

\begin{figure}[ht]
    \includegraphics[width = \columnwidth]{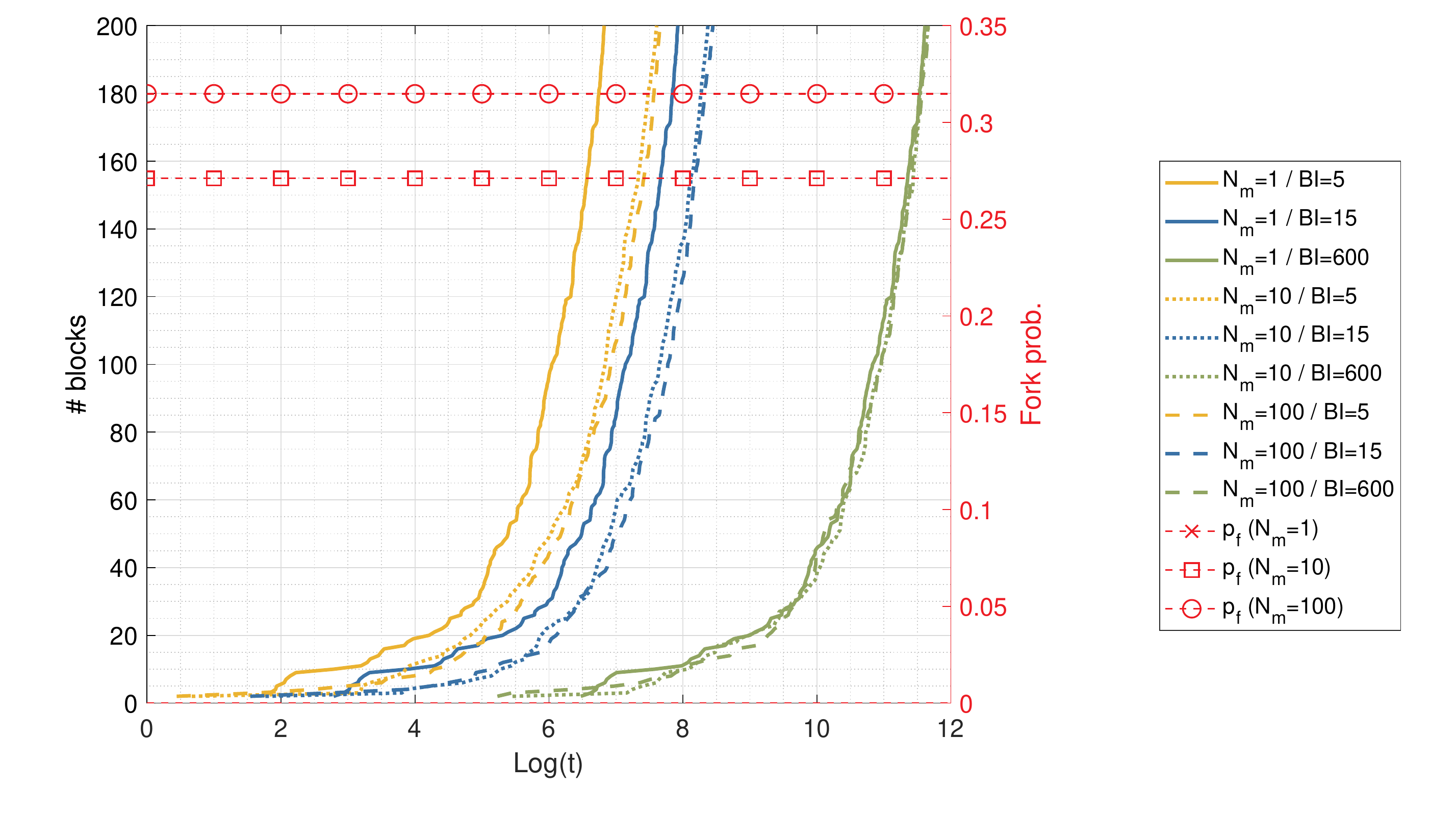}
    \caption{Blockchain delay as a function of the number of miners ($N_m$) and the block interval ($BI$). The fork probability associated with each $N_m$ is shown in red.}
    \label{fig:delay_blocks_fork_prob}
\end{figure}

As shown in Fig.~\ref{fig:delay_blocks_fork_prob}, the blockchain delay increases with the block interval ($BI$), which indicates the average time for mining a block. Notice that, in a PoW-based blockchain, the block interval is fixed by tuning the mining difficulty according to the total computational power of miners. As for the impact of $N_m$ on the delay, its effects on the delay are more noticeable for low $BI$ values. In particular, a higher fork probability is observed as $N_m$ increases, thus incurring additional delays to the FL application operating on top of the blockchain.

To optimize the performance of a blockchain, a widely adopted approach consists of finding the best block generation rate~\cite{kim2019blockchained}, which is controlled by tuning the mining difficulty. Other approaches consider optimizing the block size~\cite{wilhelmi2022end}, which fits better scenarios where the intensity of transaction arrivals depends on the nature of the application running on top of the blockchain (e.g., FL updates provided by clients).

Regarding the communication cost of BFL, it can be improved by leveraging the computational capacity of blockchain miners to speed up the FL operation. In particular, instead of including individual local models in a block, each block can bring a global model, aggregated by the miner responsible for building the block. This approach has been widely adopted in the literature (see, e.g.,~\cite{pokhrel2020federated}), and would lead to a reduced time complexity and communication cost (see Appendix~\ref{sub:bfl_proof}).

Finally, another important open aspect regarding blockchain-enabled FL lies in the implications of decentralization on the learning procedure. In this paper, we have assumed that the blockchain is perfectly shared and accessed by FL devices to carry out training, thus acting as a central orchestrating server. However, the decentralized data sharing in blockchain naturally leads to model inconsistencies, provided that different FL devices can use the information from different blocks to compute local model updates.

\subsection{Model and implementation dependencies}
\label{sub:results_cnn}
Table~\ref{tab:resultsEMNISTCNN_scavengesim} and Table~\ref{tab:resultsEMNISTCNN_iesc} report the performance of the CNN model on EMNIST (EMNISTp) datasets executed on two different platforms HW1 and HW2, respectively. Similar to the previous FFNN case, BFL is the slowest and the most energy demanding algorithm. Instead in this case, GFL reaches higher validation accuracy on EMNIST, i.e., $0.8$, but is still not able to get the performance of the other two algorithms. Moreover, using CNN, GFL is the fastest algorithm and saves up to $16\%$ of the execution time, with respect to CFL on HW1. Hence, model selection plays a key role for the algorithm performance and may facilitate the training process, as in the case of GFL. Finally, it is confirmed that the communication overhead of GFL is the lowest.

However, we report here some inconsistencies in performing energy measurements. In fact on HW1, differently from the FFNN case, CFL is the most energy efficient on EMNIST and saves $14\%$ of energy with respect to GFL. On EMNISTp, instead, the situation is different since GFL saves $15\%$ of the energy. Moreover, when using HW2 (Table~\ref{tab:resultsEMNISTCNN_iesc}), CFL results to be the most energy efficient for both EMNIST and EMNISTp. Such inconsistencies are mainly due to the fact that the average computational power consumption in CFL implementation is higher than GFL (around $103$W and $93$W on EMNIST, respectively); however GFL requires longer training time ($T_{\rm train}$). Instead in the FFNN model implementation, the average computational power consumption is higher for GFL (around $19\text{W}$ for CFL and $13\text{W}$ for GFL), but GFL requires lower training time. The reason lays mainly in the software implementations\footnote{\url{https://github.com/eliaguerra/Federated_comparison_cttc}}. In fact, CFL and BFL are based on TFF, which executes the training process for all the participating clients in parallel. Differently, GFL is based on the standard TF libraries and the training process is executed sequentially one client after the other. 

In view of the above, we state here that hardware and software implementation play a key role in the energy assessment. 
Therefore, it is essential that future research directions will focus on: i) joint optimization of federated algorithms and their software implementations, ii) definition of standard libraries for the three categories of algorithms studied in this paper, and iii) design of effective and open test platforms for experiment comparison.

\begin{table*}[ht]
\caption{CNN simulation results on HW1 and EMNIST (EMNISTp) datasets.}
\label{tab:resultsEMNISTCNN_scavengesim}
\centering
\scriptsize
    \begin{tabular}{cccccccc}
    \toprule
        & \textbf{Acc. Training} & \textbf{Acc. Validation} & \textbf{Acc. Test}   &\textbf{Time (s)}              & \textbf{Comp. Energy (\%)} & \textbf{Tot. Energy (Wh)}  & \textbf{Comm. overhead (GB)}\\ \midrule
        CFL & $0.99$ $(0.97)$   & $0.97$ $(0.9)$      & $0.96$ $(0.91)$ & $132761.16$ $(126691.59)$ & $99.04$ $(98.1)$      & $72.35$ $(36.72)$ &  $186.4$  \\
        BFL & $0.99$ $(0.97)$   & $0.97$ $(0.9)$      & $0.96$ $(0.91)$ & $138452.61$ $(132284.94)$ & $99.94$ $(99.94)$     & $1198.11$ $(1161.99)$ & $37373.2$\\
        GFL & $0.99$ $(0.67)$   & $0.81$ $(0.22)$     & $0.8$ $(0.22)$  & $113573.11$ $(106616.64)$ & $99.73$ $(99.28)$     & $83.95$ $(31.22)$  &  $93.2$\\ 
    \bottomrule
    \end{tabular}
\end{table*} 

\begin{table*}[ht]
\caption{CNN simulation results on HW2 and EMNIST (EMNISTp) datasets.}
\label{tab:resultsEMNISTCNN_iesc}
\scriptsize
\centering
\begin{tabular}{cccccccc}
    \toprule
    & \textbf{Acc. Training} & \textbf{Acc. Validation} & \textbf{Acc. Test}   & \textbf{Conv. Time (s)}              & \textbf{Comp. Energy (\%)} &\textbf{Tot. Energy (Wh)} & \textbf{Comm. overhead (GB)} \\
    \midrule
    CFL & $0.99$ $(0.97)$   & $0.97$ $(0.9)$      & $0.96$ $(0.91)$ & $125883.47$ $(124869.07)$ & $99.65$ $(99.51)$     & $201.68$ $(141.69)$  & $186.4$\\
    BFL & $0.99$ $(0.97)$   & $0.97$ $(0.9)$      & $0.96$ $(0.91)$ & $131488.87$ $(130555.3)$  & $99.95$ $(99.95)$     & $1329.65$ $(1273.95)$ & $37373.2$\\
    GFL & $0.99$ $(0.67)$   & $0.8$ $(0.22)$      & $0.8$ $(0.22)$  & $114217.66$ $(107854.86)$ & $99.93$ $(99.84)$     & $319.77$ $(143.22)$ & $93.2$\\
    \bottomrule
\end{tabular}
\end{table*}
\section{Conclusions}
\label{sec:conclusions}
Decentralized server-less federated learning is an appealing solution to overcome CFL limitations. However, finding the best approach for each scenario is not trivial due to the lack of comprehensive comparisons. In this work, we have proposed a complete overview of these techniques and evaluated them through several key performance indicators: accuracy, computational complexity, communication overhead, convergence time, and energy consumption. To do so, we have proposed a comprehensive theoretical analysis and an implementation of these algorithms.

An extensive simulation campaign has driven our analysis.
From numerical results, it emerges that GFL is the algorithm that requires less communication overhead to reach convergence. Then, CFL and GFL have similar behavior in terms of energy consumption and accuracy, but slightly differ based on the DL model adopted and the hardware used. BFL represents a viable solution for implementing decentralized learning with a high accuracy and level of security at the cost of an extra energy usage and data sharing.

Moreover, we have discussed some open issues and future research directions for the two decentralized federated methods, like the poor accuracy achieved by GFL and the blockchain overhead in BFL.  Regarding GFL, we have argued that the main drawback lies in the method used to merge model updates across the algorithm steps. We have demonstrated that with an incremental approach, the modified version of GFL is able to outperform CFL and BFL. As for BFL, we have indicated that possible optimizations go in the direction of finding the best block generation rate and block size. Moreover, we have reasoned on the possibility of reducing the time complexity by including the global model in a block, which is aggregated by the same miner building the block. In addition, we have pointed out the importance of further studies on the implication of model inconsistencies due to the fact that the blockchain cannot be perfectly shared and accessed by (all) the FL devices. 

Finally, we have argued on the key role played by the libraries used for the implementation and their influence on the energy consumption on different hardware platforms. We call for the definition of standard libraries and open test platforms to be used for research purposes.

\appendices
\section{Proofs}


\subsection{Proof of Theorem I}
\label{sub:cfl_proof}
Let us consider the procedure~\textsc{ClientUpdate}, whose time complexity is $E\left(|D_{\max}||w| + 2\frac{|D_{\max}|}{B}|w|\right)$. In fact, a single client $k$ performs the training phase on its local dataset $D_k$ along $E$ local epochs and updates the model parameters. The first operation has a time complexity of $|D_k||w|$ and the second $2\frac{|D_{k}|}{B}|w|$. The update it is executed a number of times equal to $\frac{|D_k|}{B}$, and requires a product and a sum. Each client performs $E$ local epochs, so the total cost is: 
\begin{equation}
    \sum_{k \in S_t} E\left(|D_k||w| + 2\frac{|D_k|}{B}|w|\right)
\end{equation}
To obtain an upper bound that does not depend on $k$, we can use $|D_{\max}|$ as an upper bound of $|D_k|$: 

\begin{equation}
    \begin{split}
        \sum_{k \in S_t} E&\left(|D_k||w| + 2\frac{|D_k|}{B}|w|\right) \leq \\
        &mE\left(|D_{\max}||w| + 2\frac{|D_{\max}|}{B}|w|\right).
    \end{split}
    \label{eq:cost_cupdate}
\end{equation}

We can divide the \textsc{Main} procedure in Algorithm~\ref{alg:FedAVG} into two blocks. The first, up to Line~\ref{CFL:endB1}, has a cost upper bounded by 
\begin{equation}
    mE\left(|D_{\max}||w| + 2\frac{|D_{\max}|}{B}|w|\right) + 2|w|m.
\label{eq:CFL_b1cost}
\end{equation}
In parallel every client downloads the global model, executes \textsc{ClientUpdate}, and sends the updated parameters back to the server. The download and upload operations have a time complexity proportional to $|w|$. Considering that the same procedure is repeated by $m$ clients, the upper bound in~\eqref{eq:CFL_b1cost} easily follows. The second block starts from Line~\ref{CFL:endB1}, where the server aggregates the local updates and computes the new global model. The number of arithmetical operations performed is: 
\begin{equation}
    2|w|m.
\label{eq:CFL_b2cost}\end{equation}
Combining~\eqref{eq:CFL_b1cost} and~\eqref{eq:CFL_b2cost}, and considering the number of total rounds $R$ required to reach convergence, the total cost of CFL is given by:
\begin{equation}
\begin{split}
    &R\left[mE\left(|D_{\max}||w| + 2\frac{|D_k|}{B}|w|\right) + 4m|w|\right] = \\
    &RmE|D_{\max}||w| + 2RmE\frac{|D_k|}{B}|w| + 4Rm|w|.
\end{split}
\label{eq:CFLcost_round}
\end{equation}

The first addend in~\eqref{eq:CFLcost_round} is the dominant term for the asymptotic time analysis, so this completes the proof to obtain~\eqref{eq:CFL_cost}. 

When it comes to the communications overhead of CFL, the result easily follows considering that, for each round, each clients downloads and uploads the model parameters.

\subsection{Proof of Theorem II}
\label{sub:bfl_proof}
In each algorithm's round, every client in $\mathcal{S}^t$ has to download the latest block from the closest edge server (miner) to obtain the current global model. These operations, as described before, have a cost of $|w|m$ and $2|w|m$, respectively. Then, after running the \textsc{ClientUpdate} procedure in Algorithm~\ref{alg:FLchain}, clients submit the new model weights with a cost of $|w|$. These steps are done by each node in $\mathcal{S}^t$ (in total, $m$ nodes), so the total cost is:
\begin{equation}
    \begin{split}
     m\Bigg(& 2|w|m + |w|m + E|D_{\max}||w|+ \\
    &2E\frac{|D_{\max}|}{B}|w| + |w|\Bigg).
    \end{split}
\label{eq:FLchain_training}
\end{equation}

When all the local updates have been computed, it is necessary to create a block, reach consensus throughout the mining operation, and propagate the block across all the blockchain nodes. The cost of these operations is given by: 
\begin{equation}
    2^l + m|w|N_B.
\label{eq:FLchain_block}
\end{equation}
If we combine together~\eqref{eq:FLchain_training} and~\eqref{eq:FLchain_block}, we obtain the total time complexity of the algorithm 
\begin{equation}
\begin{split}
    R\bigg(&3|w|m^2+ E|D_{\max}||w|m +\\
    & 2E\frac{|D_{\max}|}{B}|w|m + |w|m + 2^l + m|w|N_{B}\bigg).
\end{split}
\end{equation}
The dominant addends are reported in~\eqref{eq:BFL_cost}. 

The communication overhead of BFL can be easily derived from the algorithm description.

In this analysis, we considered the less efficient implementation, whereby each client has to perform the computation of the new global model given the updates in the latest block. To improve this, we can move the instruction in Line~\ref{FLchain:update} outside the \emph{for} loop and execute it before the \textsc{MineBlock} procedure. In this way, the new block has size $|w|$, since it contains only the parameters of the new model. Following the same analysis described before, the computational complexity is:
\begin{equation}
    O(R(mE|D_{\max}||w|+2^l+N_B|w|)).
\end{equation}
And the communication overhead is:
\begin{equation}
    R(2|w|m+N_B|w|).
\end{equation}

\subsection{Proof of Theorem III}
\label{sub:gfl_proof}
Let $k_i$ be a client in the sequence $[k_1, ..., k_{m}]$. Following the steps of Algorithm~\ref{alg:gfl}, three main operations are performed: 1) \textsc{Merge}, 2) \textsc{ClientUpdate} and 3) send of the model parameters to the next client of the sequence. The first one is the average of two model parameters, so its cost is $2|w|$. The cost of the second operation has already been computed in~\eqref{eq:cost_cupdate} and the cost of parameter sharing is $|w|$. By summing up these contributions we obtain:
\begin{equation}
    m\left[E\left(|D_{\max}||w| + 2\frac{|D_{\max}|}{B}|w|\right)+ 3|w|\right].
\end{equation}
This process is repeated for $R$ rounds, so the time complexity is:
\begin{equation}
    Rm\left[E\left(|D_{\max}||w| + 2\frac{|D_{\max}|}{B}|w|\right)+ 3|w|\right],
\end{equation}
where the first addend is the dominant one.

Given that each client shares its local model only with the following node in the sequence, the communication overhead is given by~\eqref{eq:GFL_comm}.

\section{Edge connection model}
\label{sec:edge_conn_model}
To compute the total duration for transmitting model weights, we assume IEEE 802.11ax channel access procedures~\cite{bellalta2016ieee}, which also include the overheads to carry out the distributed coordination function (DCF) operation. In particular, the duration of a packet transmission is defined as:
\begin{equation}
\begin{split}
    T_{\rm{Tx}} = & Rm(T_{\rm{RTS}} + T_{\rm{SIFS}} + T_{\rm{CTS}} +  
    T_{\rm{DATA}} + \\ &T_{\rm{SIFS}} + T_{\rm{ACK}} + T_{\rm{DIFS}} + T_{\rm{e}}),
\end{split}
\end{equation}
where $T_{\rm{RTS}}$ is the duration of the ready-to-send (RTS) control frame, $T_{\rm{SIFS}}$ is the short interframe space (SIFS) duration, $T_{\rm{CTS}}$ is the duration of the clear-to-send (CTS) control frame, $T_{\rm{DATA}}$ is the duration of the data payload, $T_{\rm{ACK}}$ is the duration of the acknowledgement (ACK) frame, $T_{\rm{e}}$ is the duration of an empty slot, $R$ is the number of FL rounds, and $m$ the number of participating clients.

To compute the duration of each type of IEEE 802.11ax control frame, i.e., RTS, CTS, and ACK, we compute them as:
\begin{equation}
T_{\rm{RTS/CTS/ACK}} = T_{\rm{PHY}} + \bigg\lceil \frac{L_{\rm{SF}} + L_{\rm{RTS/CTS/ACK}}}{L_{\rm{s}}} \bigg\rceil \sigma_{\rm{leg}},
\end{equation}
where $T_{\rm{PHY}}$ is the duration of the PHY preamble, $L_{\rm{SF}}$ is the length of the service field (SF), $L_{\rm{RTS/CTS/ACK}}$ is the length of the control frame, $L_{\rm{s}}$ is the length of an orthogonal frequency division multiplexing (OFDM) symbol, and $\sigma_{\rm{leg}}$ is the duration of a legacy OFDM symbol.

As for the duration of the data payload, it is computed as:
\begin{equation}
    T_{\rm{DATA}} = T_{\rm{HE-SU}} + \bigg\lceil \frac{L_{\rm{SF}} + L_{\rm{MAC}} + L_{\rm{DATA}}}{L_{\rm{s}}} \bigg\rceil \sigma,
\end{equation}
where $T_{\rm{HE-SU}}$ is the duration of the high-efficiency (HE) single-user field, $L_{\rm{MAC}}$ is the length of the MAC header, $L_{\rm{DATA}}$ is the length of a single data packet (in our case, it matches with the model size, $S_w$), and $\sigma$ is the duration of an OFDM symbol. The number of bits per OFDM symbol will vary, so as the effective data rate, based on the employed modulation and coding scheme (MCS), which depends on the transmission power used.


\bibliography{references}
\bibliographystyle{IEEEtran}

\clearpage
\begin{IEEEbiography}
[{\includegraphics[width=1in,height=1.25in,clip,keepaspectratio]{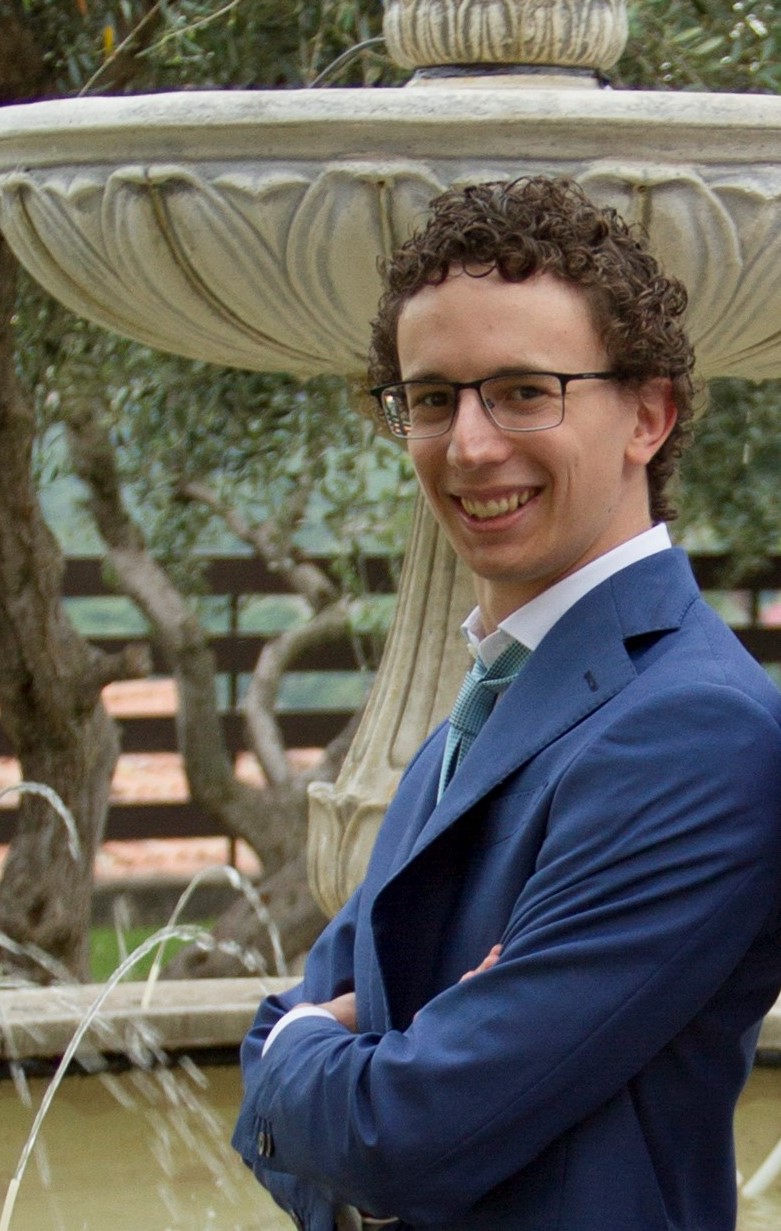}}]
    {Elia Guerra} received his master's degree in Computer Engineering at the University of Padova (Italy) in 2021. Prior to this, he got his bachelor's degree in Information Engineering in 2019. During his studies, he developed a passion for Machine Learning and Algorithms.
    He is a Ph.D. student at the Technical University of Catalonia (UPC) and he is currently working at CTTC for the GREENEDGE (MSCA ETN) project. His main research lines are distributed/decentralized and sustainable machine learning algorithms.
\end{IEEEbiography}

\begin{IEEEbiography}
[{\includegraphics[width=1in,height=1.25in,clip,keepaspectratio]{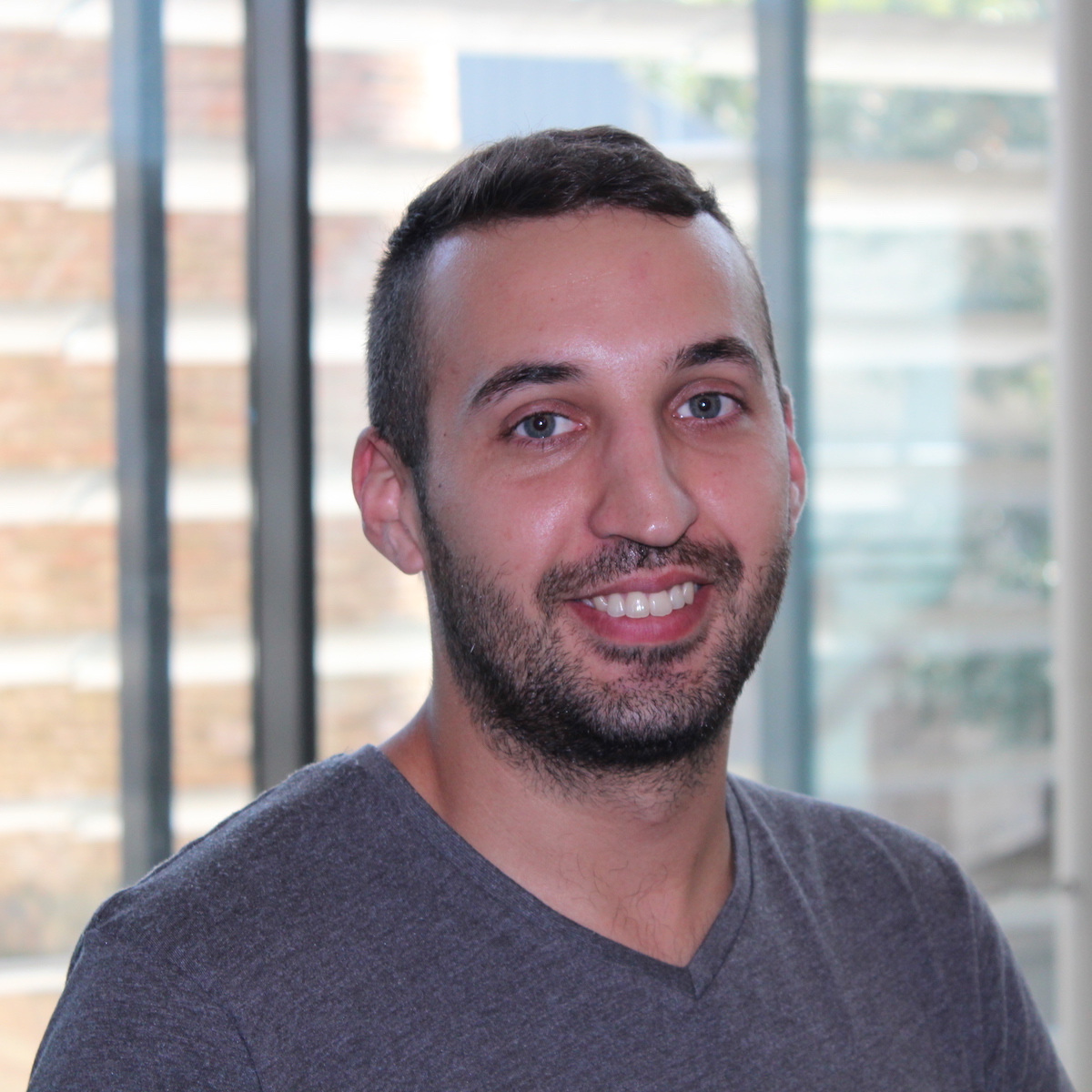}}]
    {Francesc Wilhelmi}holds a Ph.D. in Information and Communication Technologies (2020), from Universitat Pompeu Fabra (UPF). Previously, he obtained a B.Sc. degree in Telematics Engineering (2015) and an M.Sc. in Intelligent and Interactive Systems (2016), also from the UPF. He is currently working as a researcher at Nokia Bell Labs.
\end{IEEEbiography}

\begin{IEEEbiography}
    [{\includegraphics[width=1in,height=1.25in,clip,keepaspectratio]{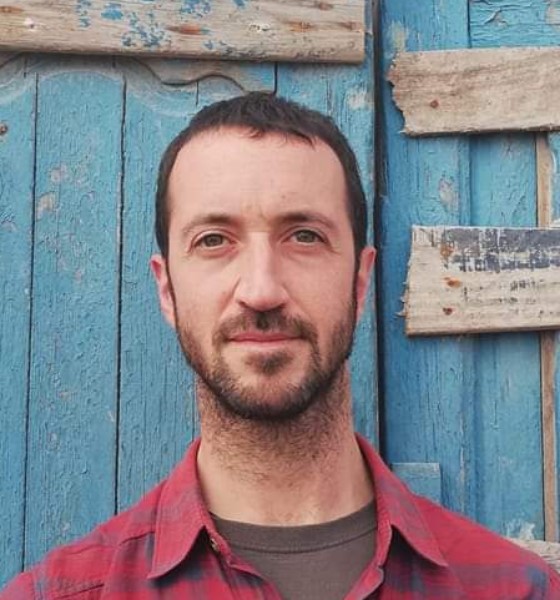}}]
    {Marco Miozzo} received his M.Sc. degree in Telecommunication Engineering from the University of Ferrara (Italy) in 2005 and the Ph.D. from the Technical University of Catalonia (UPC) in 2018. In June 2008 he joined the Centre Tecnologic de Telecomunicacions de Catalunya (CTTC). In CTTC he has been involved in several EU founded projects. He participated in several R\&D projects, among them SCAVENGE, 5G-Crosshaul, Flex5Gware and SANSA, working on environmental sustainable mobile networks with energy harvesting capabilities through learning techniques. Currently he is collaborating with the EU founded H2020 GREENEDGE (MSCA ETN) and SONATA (CHIST-ERA). His main research interests are: sustainable mobile networks, green wireless networking, energy harvesting, multi-agent systems, machine learning, green AI, explainable AI.

\end{IEEEbiography}

\begin{IEEEbiography}    [{\includegraphics[width=1in,height=1.25in,clip,keepaspectratio]{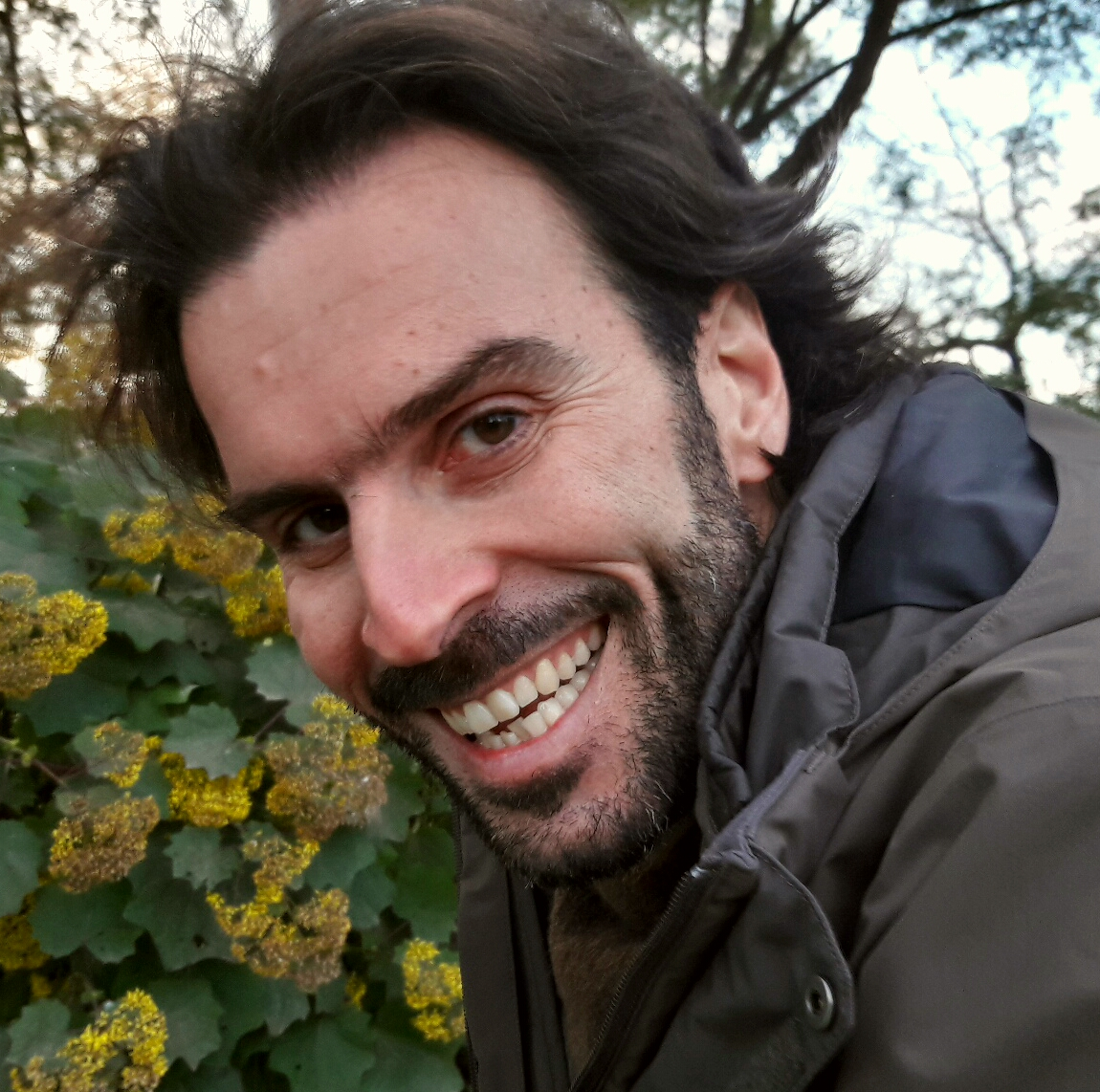}}]{Paolo Dini} received M.Sc. and Ph.D. from the Universit`a di Roma La Sapienza, in 2001 and 2005, respectively. He is currently a Senior Researcher with the Centre Tecnologic de Telecomunicacions de Catalunya (CTTC). His current research interests include sustainable networking and computing, distributed optimization and optimal control, machine learning, multi-agent systems and data analytics. His research activity is documented in almost 90 peer-reviewed scientific journals and international conference papers. He received two awards from the Cisco Silicon Valley Foundation for his research on heterogeneous mobile networks, in 2008 and 2011, respectively. He has been involved in more than 25 research projects. He is currently the Coordinator of CHIST-ERA SONATA project on sustainable computing and communication at the edge and the Scientific Coordinator of the EU H2020 MSCA Greenedge European Training Network on edge intelligence and sustainable computing. He serves as a TPC in many international conferences and workshops and as a reviewer for several scientific journals of the IEEE, Elsevier, ACM, Springer, Wiley.
\end{IEEEbiography}

\end{document}